\documentclass[11pt]{article}

\usepackage[letterpaper,margin=1in]{geometry}
\usepackage{amsmath,amssymb,amsthm,mathtools,bm}
\usepackage{algorithm}
\usepackage{algorithmic}
\usepackage{booktabs}
\usepackage{xcolor}
\usepackage{kotex}
\usepackage{graphicx}
\usepackage[numbers,sort&compress]{natbib}
\usepackage[colorlinks=true,linkcolor=blue,citecolor=blue,urlcolor=blue]{hyperref}
\usepackage[nameinlink,noabbrev]{cleveref}

\graphicspath{{figures/}}
\DeclareGraphicsExtensions{.eps}

\newcommand{\R}{\mathbb{R}}
\newcommand{\E}{\mathbb{E}}
\newcommand{\Prob}{\mathbb{P}}
\newcommand{\Sset}{\mathcal{S}}
\newcommand{\Aset}{\mathcal{A}}
\newcommand{\Vsub}{\mathcal{V}}
\newcommand{\Asub}{\mathcal{A}_0}
\newcommand{\Hfam}{\mathcal{H}}
\newcommand{\one}{\mathbf{1}}
\newcommand{\co}{\operatorname{co}}
\newcommand{\range}{\operatorname{range}}
\newcommand{\diag}{\operatorname{diag}}
\newcommand{\col}{\operatorname{col}}
\newcommand{\argmax}{\operatorname*{arg\,max}}
\newcommand{\norm}[1]{\left\lVert #1 \right\rVert}

\newtheorem{definition}{Definition}[section]
\newtheorem{lemma}{Lemma}[section]
\newtheorem{proposition}{Proposition}[section]
\newtheorem{theorem}{Theorem}[section]
\newtheorem{remark}{Remark}[section]
\newtheorem{example}{Example}[section]

\crefname{algorithm}{Algorithm}{Algorithms}
\Crefname{algorithm}{Algorithm}{Algorithms}
\crefname{appendix}{Appendix}{Appendices}
\Crefname{appendix}{Appendix}{Appendices}

\title{Spectral Analysis of Dueling Q-Learning}
\author{Donghwan Lee\\
\small Department of Electrical Engineering, Korea Advanced Institute of Science and Technology (KAIST)\\
\small Daejeon 34141, South Korea (email: \texttt{donghwan@kaist.ac.kr})}
\date{}

\begin{document}
\maketitle

\begin{abstract}
Q-learning is a fundamental algorithm in reinforcement learning (RL) for solving discounted Markov decision processes (MDPs) when the transition kernel is unknown. The deep Q-network (DQN) extends Q-learning by using a deep neural network for Q-function approximation, which makes Q-learning applicable to more practical high-dimensional problems. Dueling Q-learning decomposes the Q-function into a value function and an advantage function and learns the two components jointly, which can improve learning efficiency. However, the theoretical understanding of dueling Q-learning is still limited. Recent work has initiated an analysis of tabular dueling Q-learning, but existing guarantees focus on a regularized formulation and leave the pure tabular update less completely understood. This paper strengthens that line of analysis by adding a direct interpretation of the centered tabular decomposition and by establishing convergence guarantees for the unregularized, unprojected constant step-size recursion. In particular, we derive an exact switching linear system representation for deterministic dueling Q-learning and a finite-time error bound in expectation for the sampled stochastic version. The analysis clarifies how the value and advantage updates act as different gains on the action-common (value function) and action-differential (advantage function) components of the Q-function.
\end{abstract}

\section{Introduction}
Q-learning~\cite{watkins1992q} is a foundational algorithm in reinforcement learning (RL)~\cite{sutton1998reinforcement} for solving discounted Markov decision processes (MDPs)~\cite{puterman1994markov} with unknown transition kernels. The deep Q-network (DQN)~\cite{mnih2015humanlevel} extends Q-learning by using a deep neural network for Q-function approximation, which makes value-based RL applicable to high-dimensional problems in which a tabular representation is not practical. The dueling network architecture~\cite{wang2016dueling} further modifies DQN by separating the Q-function approximation into a value stream and an advantage stream. This decomposition can improve learning efficiency because the value component can share state-wise information across actions while the advantage component captures action-dependent deviations. Despite its empirical usefulness, the theoretical understanding of dueling Q-learning is still much less complete than that of standard tabular Q-learning.

A recent study of action-value temporal-difference methods~\cite{daley2025analysis} that learn state values formalized tabular versions of dueling methods and introduced regularized dueling Q-learning. That work clarified important aspects of dueling Q-learning (also called AV-learning) and provided a theoretical solution analysis for a regularized formulation of dueling Q-learning. However, the pure tabular dueling Q-learning recursion, without a regularization term or projection step, still leaves room for a more direct convergence analysis. The present paper addresses this gap by analyzing the unregularized dueling Q-learning recursion with constant step-sizes and by deriving a finite-time error bound in expectation for its sampled stochastic version.

The analysis begins by interpreting dueling Q-learning through an orthogonal common and differential decomposition of the tabular Q-function. For a tabular Q-function $Q$, define the state-wise mean
\begin{align*}
 \bar V_Q(s):=\frac{1}{|\Aset|}\sum_{b\in\Aset}Q(s,b),
\end{align*}
and identify this state vector with its action-repeated lift
\begin{align*}
 (\Pi Q)(s,a):=\bar V_Q(s),\qquad s\in\Sset,\ a\in\Aset.
\end{align*}
Thus $\Pi Q$ is the projection of $Q$ onto the subspace of vectors that are constant across actions in each state. The action-differential component is the complementary centered projection
\begin{align*}
 A(s,a)= Q(s,a)- (\Pi Q)(s,a)=Q(s,a)-\frac{1}{|\Aset|}\sum_{b\in\Aset}Q(s,b).
\end{align*}
Therefore, $Q$ is written as the sum of two components, denoted by $V$ and $A$, where $V=\Pi Q$ is the projected action-common component and $A$ is the centered action-differential component. In this paper, $V$ and $A$ are component coordinates induced by $Q$; they are motivated by, but are not identified with, the conventional policy-dependent value and advantage functions unless explicitly stated. This interpretation leads to a deterministic update in which the Bellman residual is multiplied by different gains on the two subspaces: the value update uses the action-common gain $|\Aset|\alpha$, whereas the advantage update uses the action-differential gain $\beta$. Based on this decomposition, we develop a centered AV-learning update (also dueling Q-learning) and prove a convergence condition for it.

The algorithm studied here is close to, but not identical in presentation to, the AV-learning update discussed by Daley et al.~\cite{daley2025analysis}. Their formulation uses the same scalar gain for the value and advantage updates and maintains an advantage variable whose state-wise mean does not affect the reconstructed Q-function. In contrast, we work with the centered representative $A=\Pi_\perp Q$, allow two gains $\alpha$ and $\beta$, and keep the identity $Q=V+A$ throughout the analysis. We also introduce a Q-only version of the same recursion. The centered AV-learning update and the Q-only update are shown to generate exactly the same Q-iterates under the centered identification.

We first study a deterministic version of the recursion to isolate the main mechanism, and then we treat the sampled stochastic RL version. The analysis uses switching linear system (SLS) theory from control~\cite{liberzon2003switching,lin2009stability,shorten2007stability}. An SLS is a dynamical model in which one matrix from a prescribed family is applied at each iteration, and the active matrix may change over time according to a switching signal. Its worst-case exponential growth rate is characterized by the joint spectral radius (JSR)~\cite{rota1960note,jungers2009joint}, which is the asymptotic maximal growth rate of finite products of matrices from the switching family. If the JSR is less than one, the SLS is exponentially stable under arbitrary switching. We express the deterministic dueling Q-learning (or centered AV-learning) error recursion as an SLS and obtain convergence from a JSR condition. For the sampled recursion with constant step-sizes, the same construction yields a finite-time error bound in expectation under a JSR condition. The conclusion is convergence to a first-moment neighborhood of the optimal Q-function, and the size of this neighborhood goes to zero as the common scalar gain goes to zero. This paper treats only the independent and identically distributed (i.i.d.) sampling case to keep the formulas transparent. The same conditional-mean and noise decomposition can be combined with Markovian-observation stochastic-approximation arguments, as in \cite{lee2026lyapunovcertified}, to extend the setting beyond i.i.d. samples.

\section{Related Work}\label{sec:related_works}
This section positions the present paper relative to existing value--advantage and value-augmented action-value methods. Baird's advantage updating~\cite{baird1994advantage} is an early approach that stores state values and state-action advantages separately. It is conceptually related to the decomposition used here because it separates a state-common quantity from action-dependent deviations. However, advantage updating is not identical to the centered AV-learning (dueling Q-learning) analyzed in this paper, and it did not establish a theoretical convergence analysis for the centered recursion studied here.

Wang et al.~\cite{wang2016dueling} introduced the dueling network architecture for deep Q-networks (DQNs), in which separate value and advantage streams are combined to produce Q-values. That work motivates the value--advantage parameterization and demonstrates empirical benefits in deep reinforcement learning, but it does not analyze convergence of the corresponding tabular recursion.

A related line of work learns state values together with action values. QV($\lambda$)-learning~\cite{wiering2005qvlambda} and the QV family~\cite{wiering2009qvfamily} maintain Q-values together with a separate state-value function and include both on-policy and off-policy variants. These algorithms show that augmenting action-value learning with state-value estimates can improve empirical behavior. Their mechanism is nevertheless different from the present one: QV methods use separate estimates for the Q-function and the value function in their bootstrap targets, whereas the recursion studied here reconstructs $Q=V+A$ through an orthogonal common/differential decomposition. Moreover, these QV studies are primarily empirical and do not provide the centered projection-based convergence analysis pursued here.

More recent work has made the connection between value--advantage learning and dueling architectures explicit. Tang et al.~\cite{tang2023valearning} proposed VA-learning, which directly learns a value function and an advantage function and relates the resulting update to dueling architectures. Daley et al.~\cite{daley2025analysis} analyzed action-value temporal-difference methods that learn state values, formalized AV-learning and QV-learning families, derived a tabular dueling Q-learning update, and introduced regularized dueling Q-learning. These papers are the closest to the present algorithmic setting. However, the algorithmic structure in~\cite{tang2023valearning} differs from the centered recursion considered here. The formulation in~\cite{daley2025analysis} is closer, but it uses regularization to handle the dueling decomposition, whereas the present analysis establishes convergence for the unregularized and unprojected centered recursion. In addition, we provide a finite-time error bound in expectation for the sampled recursion and give a complementary SLS and joint spectral radius (JSR) interpretation of the value--advantage dynamics. The analysis of Daley et al.~\cite{daley2025analysis} does not establish the same SLS/JSR interpretation or a convergence guarantee for the unregularized centered recursion considered here.

\section{Preliminaries}\label{sec:preliminaries}
\subsection{Notation}\label{sec:notation}
The set of real numbers is denoted by $\R$; $\R^m$ is the $m$-dimensional Euclidean space; and $\R^{m\times n}$ is the set of all $m\times n$ real matrices. For a matrix $A$, $A^\top$ denotes its transpose. The identity matrix is denoted by $I$. For a finite state set $\Sset$ and a finite action set $\Aset$, $e_s$ and $e_a$ denote the standard basis vectors associated with $s\in\Sset$ and $a\in\Aset$, respectively. For the state-action index $i=(s,a)$, write $e_i=e_a\otimes e_s$ in the action-block ordering. More generally, standard basis vectors in other Euclidean spaces use the same notation when the dimension is clear from context, and $\otimes$ denotes the Kronecker product. For a finite set $\Sset$, $|\Sset|$ denotes its cardinality. For finite tabular state and action sets, set $n:=|\Sset||\Aset|$. Moreover, we write $\Delta_m:=\left\{q\in\R^m:q_i\ge 0,\ \sum_{i=1}^m q_i=1\right\}$ for the probability simplex in $\R^m$. For a finite matrix family $\Hfam=\{A_1,\ldots,A_N\}$, $\co(\Hfam):=\left\{\sum_{i=1}^N \lambda_i A_i:\lambda_i\ge 0,\ \sum_{i=1}^N \lambda_i=1\right\}$ denotes its convex hull. We also use standard matrix notation that appears repeatedly below. For a vector $x$, $\norm{x}_2$ is the Euclidean norm. For a square matrix $A$, $\rho(A)$ denotes its standard spectral radius. After the matrix-family radius is introduced below, the notation $\rho(\Hfam)$ is used for that quantity when the argument is a switching family. For a matrix $B$, $\range(B)$ denotes its column space, $B\succ0$ means that $B$ is symmetric positive definite, and $\lambda_{\max}(B)$ denotes the largest eigenvalue when $B$ is symmetric. For square matrices $A$ and $B$ of the same size, we write $A\sim B$ when they are similar, that is, when $B=S^{-1}AS$ for some nonsingular matrix $S$. Expectations are denoted by $\E[\cdot]$.

\subsection{Switched linear systems}\label{sec:switching_systems}
The stability certificates used later are stated in the language of switching systems \cite{liberzon2003switching,lin2009stability,shorten2007stability}. Therefore, we first recall the basic switching system model before specializing it to the Bellman-induced switching system model. Consider the discrete-time switching affine system (SAS)
\[
    x_{k+1}=A_{\sigma_k}x_k+b_{\sigma_k},
\]
where each index $i\in\{1,2,\ldots,M\}$, equivalently each affine pair $(A_i,b_i)$, is called a mode, and $\sigma_k$ is the switching signal that selects the active mode at time $k$. The matrix $A_{\sigma_k}$ is selected from the prescribed family $\Hfam:=\{A_1,A_2,\ldots,A_M\}$, which is called a switching family; $b_{\sigma_k}$ is a mode-dependent affine term. When $b_{\sigma_k}=0$, the deterministic part reduces to an SLS, $x_{k+1}=A_{\sigma_k}x_k$. An SLS is exponentially stable under arbitrary switching if there exist constants $C\ge 1$ and $\eta\in(0,1)$ such that $\norm{A_{\sigma_{k-1}}\cdots A_{\sigma_0}x}_2\le C\eta^k\norm{x}_2$ for every horizon $k\ge0$, every initial state $x\in\R^m$, and every switching sequence. The worst-case exponential rate of the SLS family is characterized by the joint spectral radius (JSR) \cite{rota1960note,jungers2009joint}, defined as follows.

\begin{definition}\label{def:jsr}
For a bounded set of matrices $\Hfam\subset\R^{m\times m}$, its JSR is
\[
    \rho(\Hfam):=\lim_{k\to\infty}\sup_{A_1,\ldots,A_k\in\Hfam}\norm{A_k\cdots A_1}^{1/k}.
\]
\end{definition}

The JSR is independent of the chosen submultiplicative norm \cite{rota1960note,jungers2009joint}. When $\Hfam$ is finite, the supremum for each fixed product length is a maximum over products generated by matrices in $\Hfam$. For a finite family $\Hfam$, the notation $\rho(\co(\Hfam))$ means the JSR computed when each factor in a product is allowed to be any convex combination of matrices in $\Hfam$.

\subsection{JSR and Lyapunov certificates}\label{sec:joint_spectral_radius}
This subsection introduces the Lyapunov viewpoint used to analyze convergence of the algorithms below. A common Lyapunov function for $\Hfam$ is a positive definite function that decreases along every mode. The following finite-family piecewise-quadratic construction \cite{lee2026lyapunovcertified,hushenzhang2010generating} is the Lyapunov certificate used in the deterministic arguments.

\begin{lemma}\label{lem:common_lyapunov_construction}
Let $\Hfam=\{A_1,A_2,\ldots,A_M\}\subset\R^{m\times m}$ and fix $\varepsilon>0$ such that $\beta_\varepsilon:=\rho(\Hfam)+\varepsilon\in(0,1)$. For a word $\sigma=(\sigma_1,\ldots,\sigma_k)\in\{1,\ldots,M\}^k$, write
\[
    A_\sigma:=A_{\sigma_k}\cdots A_{\sigma_1},
\]
with the convention that the empty word gives $A_\sigma=I$. Define
\[
    V_\varepsilon^\infty(x):=\sum_{k=0}^{\infty}\beta_\varepsilon^{-2k}\max_{\sigma\in\{1,\ldots,M\}^k}\norm{A_\sigma x}_2^2,
    \qquad x\in\R^m.
\]
Then $V_\varepsilon^\infty$ is finite for every $x$, and there exists $C_\varepsilon>0$ such that
\[
    \norm{x}_2^2\le V_\varepsilon^\infty(x)\le C_\varepsilon\norm{x}_2^2,\qquad \forall x\in\R^m.
\]
The function $p_\varepsilon(x):=\sqrt{V_\varepsilon^\infty(x)}$ is a norm on $\R^m$, and every mode satisfies
\[
    p_\varepsilon(A_i x)\le \beta_\varepsilon p_\varepsilon(x),\qquad \forall x\in\R^m,
    \qquad i=1,
    \ldots,M.
\]
\end{lemma}
\begin{proof}
The proof is the standard finite-family JSR Lyapunov construction; see \cite{lee2026lyapunovcertified} and the construction in \cite{hushenzhang2010generating}. We omit the details.
\end{proof}

Throughout the sequel, whenever this construction is applied to a switching family with JSR less than one, we call the resulting $V_\varepsilon^\infty$ a JSR Lyapunov function for that family, and we call the associated norm $p_\varepsilon$ a JSR Lyapunov norm. The following lemma is the common bridge from a JSR bound to convergence of the corresponding error recursion.

\begin{lemma}\label{lem:standard_jsr_convergence}
Let $\Hfam=\{A_1,\ldots,A_M\}\subset\R^{n\times n}$ be finite and suppose $\rho(\Hfam)<1$. Consider any recursion with $x_0\in\R^n$ of the form
\[
    x_{k+1}=A_kx_k,\qquad A_k\in\co(\Hfam),\qquad k\in\{0,1,\ldots\}.
\]
Then, for every $\varepsilon>0$ such that $\beta_\varepsilon:=\rho(\Hfam)+\varepsilon<1$, the Lyapunov function $V_\varepsilon^\infty$ and norm $p_\varepsilon$ from~\Cref{lem:common_lyapunov_construction}, applied to $\Hfam$, satisfy
\[
    V_\varepsilon^\infty(x_{k+1})\le \beta_\varepsilon^2 V_\varepsilon^\infty(x_k),
    \qquad
    p_\varepsilon(x_{k+1})\le \beta_\varepsilon p_\varepsilon(x_k).
\]
Consequently, if $C_\varepsilon$ is the constant from~\Cref{lem:common_lyapunov_construction}, then
\[
    p_\varepsilon(x_k)\le \beta_\varepsilon^k p_\varepsilon(x_0),
    \qquad
    \norm{x_k}_2\le \beta_\varepsilon^k p_\varepsilon(x_0)\le \sqrt{C_\varepsilon}\,\beta_\varepsilon^k\norm{x_0}_2,
\]
and hence $x_k\to0$.
\end{lemma}
\begin{proof}
Let $p_\varepsilon$ be the norm from~\Cref{lem:common_lyapunov_construction}. If $A\in\co(\Hfam)$, then $A=\sum_{i=1}^M\lambda_iA_i$ for some $\lambda_i\ge0$ with $\sum_i\lambda_i=1$. Since $p_\varepsilon$ is a norm and $p_\varepsilon(A_i x)\le \beta_\varepsilon p_\varepsilon(x)$ for every mode $A_i$, we have
\[
    p_\varepsilon(Ax)
    =p_\varepsilon\left(\sum_{i=1}^M\lambda_iA_ix\right)
    \le \sum_{i=1}^M\lambda_i p_\varepsilon(A_ix)
    \le \beta_\varepsilon p_\varepsilon(x).
\]
Applying this with $A=A_k$ gives the one-step inequalities for $p_\varepsilon$ and $V_\varepsilon^\infty=p_\varepsilon^2$. Iteration gives $p_\varepsilon(x_k)\le \beta_\varepsilon^k p_\varepsilon(x_0)$. The Euclidean bound follows from the norm-equivalence estimate in~\Cref{lem:common_lyapunov_construction}.
\end{proof}

The rest of the paper uses this JSR implication as the main convergence certificate. Once an error recursion has been written as an SLS whose mode family has JSR below one, convergence and an exponential rate bound follow from~\Cref{lem:standard_jsr_convergence}. The same matrices identify the drift term in sampled stochastic RL recursions. For the sampled version, a JSR bound below one for the associated conditional-mean switching family gives the drift part needed for stochastic convergence and finite-time error bounds, as in \cite{lee2026lyapunovcertified}. We note that this paper treats only the independent and identically distributed (i.i.d.) sampling case to keep the formulas transparent. The same conditional-mean and noise decomposition can be combined with Markovian-observation stochastic-approximation arguments, as in \cite{lee2026lyapunovcertified}, to extend the setting beyond i.i.d. samples.

\subsection{Discounted Markov decision processes}\label{sec:discounted_mdps}
We consider a finite discounted Markov decision process (MDP), the standard model for control in reinforcement learning (RL) \cite{puterman1994markov,bertsekas1996neuro,sutton1998reinforcement}, with state-space $\Sset=\{1,\ldots,|\Sset|\}$, action space $\Aset=\{1,\ldots,|\Aset|\}$, transition probability $P(s'\mid s,a)$, real-valued one-step reward $r(s,a,s')$, and discount factor $\gamma\in(0,1)$. The finite state and action sets and the real-valued reward function imply the uniform reward bound
\begin{equation}\label{eq:reward_bound}
    R_{\max}:=\max_{s\in\Sset,\ a\in\Aset,\ s'\in\Sset}|r(s,a,s')|<\infty.
\end{equation}
The expected reward is $R(s,a):=\sum_{s'\in\Sset}P(s'\mid s,a)r(s,a,s')$. State-action functions are viewed as vectors in $\R^n$ using the action-block ordering $(1,1),(2,1),\ldots,(|\Sset|,1),(1,2),(2,2),\ldots,(|\Sset|,|\Aset|)$. All matrices and vectors indexed by state-action pairs use this ordering. Define
\[
    P:=\begin{bmatrix} P_1\\ \vdots\\ P_{|\Aset|}\end{bmatrix}\in\R^{n\times |\Sset|},
    \qquad
    R:=\begin{bmatrix} R(\cdot,1)\\ \vdots\\ R(\cdot,|\Aset|)\end{bmatrix}\in\R^n,
\]
where $P_a=P(\cdot\mid\cdot,a)\in\R^{|\Sset|\times |\Sset|}$. Let $\Theta$ denote the set of deterministic stationary policies $\pi:\Sset\to\Aset$. For any stochastic policy $\mu:\Sset\to\Delta_{|\Aset|}$, define
\[
    \Gamma^\mu:=
    \begin{bmatrix}
        \mu(1)^\top\otimes e_1^\top\\
        \mu(2)^\top\otimes e_2^\top\\
        \vdots\\
        \mu(|\Sset|)^\top\otimes e_{|\Sset|}^\top
    \end{bmatrix}\in\R^{|\Sset|\times n}.
\]
Here $\Delta_{|\Aset|}$ is the probability simplex in $\R^{|\Aset|}$, so $\mu(s)$ is a probability vector over actions at state $s$. For a deterministic policy $\pi\in\Theta$, the same notation $\Gamma^\pi$ is used by identifying $\pi(s)$ with its one-hot encoding.
\begin{lemma}\label{lem:policy_selector_convex_hull}
For every stochastic policy $\mu:\Sset\to\Delta_{|\Aset|}$,
\[
    P\Gamma^\mu\in\co\{P\Gamma^\pi:\pi\in\Theta\}.
\]
Moreover, $P\Gamma^\mu$ is row-stochastic and satisfies $P\Gamma^\mu\one=\one$.
\end{lemma}
\begin{proof}
For each deterministic policy $\pi\in\Theta$, define the convex weight
\[
    \lambda_\pi:=\prod_{s\in\Sset}\mu(\pi(s)\mid s).
\]
Then $\lambda_\pi\ge0$ and
\[
    \sum_{\pi\in\Theta}\lambda_\pi
    =\prod_{s\in\Sset}\sum_{a\in\Aset}\mu(a\mid s)=1.
\]
For fixed $s\in\Sset$ and $a\in\Aset$,
\[
    \sum_{\pi\in\Theta:\,\pi(s)=a}\lambda_\pi
    =\mu(a\mid s)\prod_{\bar s\in\Sset\setminus\{s\}}\sum_{b\in\Aset}\mu(b\mid \bar s)
    =\mu(a\mid s).
\]
Therefore the $s$th row of $\sum_{\pi\in\Theta}\lambda_\pi\Gamma^\pi$ equals the $s$th row of $\Gamma^\mu$, and hence
\[
    \Gamma^\mu=\sum_{\pi\in\Theta}\lambda_\pi\Gamma^\pi.
\]
Multiplying by $P$ gives the stated convex-hull inclusion. Since $P$ and $\Gamma^\mu$ are entrywise nonnegative and map the all-ones vector to the all-ones vector of the appropriate dimension, $P\Gamma^\mu$ is row-stochastic and $P\Gamma^\mu\one=\one$.
\end{proof}
For $Q\in\R^n$, let us define
\[
    V_Q(s):=\max_{a\in\Aset}Q(s,a),
    \qquad
    V_Q:=(V_Q(1),\ldots,V_Q(|\Sset|))^\top.
\]
The Bellman optimality operator is $F(Q):=R+\gamma P V_Q$. Its unique fixed point is the optimal tabular action-value function $Q^\star$, which satisfies
\begin{equation}\label{eq:Qstar_tabular_optimal}
    Q^\star=F(Q^\star)=R+\gamma P V_{Q^\star}.
\end{equation}
For a tie-broken greedy policy $\pi_Q$ satisfying
\[
    \pi_Q(s)\in\argmax_{a\in\Aset} Q(s,a),\qquad s\in\Sset,
\]
one has $V_Q=\Gamma^{\pi_Q}Q$ and
\[
    F(Q)=R+\gamma P\Gamma^{\pi_Q}Q.
\]
The following lemma is used repeatedly to write Bellman differences as linear maps depending on stochastic policies; see, e.g.,~\cite{lee2026lyapunovcertified}.
\begin{lemma}\label{lem:bellman_difference_selector}
For any two vectors $Q,\overline Q\in\R^n$, there exists a stochastic policy $\mu_{Q,\overline Q}:\Sset\to\Delta_{|\Aset|}$ such that
\[
    V_Q-V_{\overline Q}=\Gamma^{\mu_{Q,\overline Q}}(Q-\overline Q).
\]
Moreover, the selector can be chosen as a Borel function of $(Q,\overline Q)$. Consequently,
\[
    F(Q)-F(\overline Q)=\gamma P\Gamma^{\mu_{Q,\overline Q}}(Q-\overline Q).
\]
In particular, if $Q$ and $\overline Q$ are $\mathcal G$-measurable random vectors for a sigma-field $\mathcal G$, then $\mu_{Q,\overline Q}$ can be chosen $\mathcal G$-measurable.
\end{lemma}
\begin{proof}
Fix a state $s$ and write
\[
    \Delta_a(s):=Q(s,a)-\overline Q(s,a),
    \qquad
    h(s):=V_Q(s)-V_{\overline Q}(s).
\]
The elementary inequality
\[
    \min_{a\in\Aset}\Delta_a(s)\le h(s)\le \max_{a\in\Aset}\Delta_a(s)
\]
follows from the variational characterization of the maximum. Choose lexicographic minimizers and maximizers
\[
    a_{\min}(s)\in\arg\min_{a\in\Aset}\Delta_a(s),
    \qquad
    a_{\max}(s)\in\arg\max_{a\in\Aset}\Delta_a(s),
\]
and set $\Delta_{\min}(s):=\Delta_{a_{\min}(s)}(s)$ and $\Delta_{\max}(s):=\Delta_{a_{\max}(s)}(s)$. If $\Delta_{\min}(s)=\Delta_{\max}(s)$, assign probability one to $a_{\min}(s)$. Otherwise, define
\[
    \mu_{Q,\overline Q}(a_{\min}(s)\mid s)
    :=\frac{\Delta_{\max}(s)-h(s)}{\Delta_{\max}(s)-\Delta_{\min}(s)},
\]
\[
    \mu_{Q,\overline Q}(a_{\max}(s)\mid s)
    :=\frac{h(s)-\Delta_{\min}(s)}{\Delta_{\max}(s)-\Delta_{\min}(s)},
\]
and assign zero probability to all other actions. These weights are nonnegative and sum to one, and they satisfy
\[
    \sum_{a\in\Aset}\mu_{Q,\overline Q}(a\mid s)\Delta_a(s)=h(s).
\]
Since this holds for every state $s$, we get $V_Q-V_{\overline Q}=\Gamma^{\mu_{Q,\overline Q}}(Q-\overline Q)$. The displayed identity for $F(Q)-F(\overline Q)$ follows from the definition of $F$. The lexicographic tie-breaking construction is piecewise continuous on finitely many regions, hence Borel measurable as a function of $(Q,\overline Q)$; the final measurability statement follows by composition.
\end{proof}

We use $d$ to denote a state-action sampling distribution on $\Sset\times\Aset$. In the i.i.d. observation model, $d$ is the sampling distribution of $(s_k,a_k)$; in the Markovian observation model, $d$ is the stationary state-action distribution of the behavior-induced chain. Throughout the paper, we assume that the sampling distribution satisfies $d(s,a)>0$ for every $(s,a)\in\Sset\times\Aset$. This distribution forms the full-support state-action weight vector $d\in\R^n$ with $d^\top\one=1$, and we define
\begin{equation}\label{eq:sampling_weight_extrema}
    D:=\diag(d),\qquad d_{\min}:=\min_{(s,a)\in\Sset\times\Aset}d(s,a),\qquad d_{\max}:=\max_{(s,a)\in\Sset\times\Aset}d(s,a).
\end{equation}

\section{Standard deterministic Q-learning}
This section first reviews the deterministic form of standard tabular Q-learning \cite{watkins1992q}. This baseline is the reference recursion for the dueling Q-learning analysis below. Standard deterministic Q-learning with step-size $\eta$ is
\begin{equation}\label{eq:standard_tabular_Q}
    Q_{k+1}=Q_k+\eta D\bigl(F(Q_k)-Q_k\bigr).
\end{equation}
Let $Q^\star$ be the optimal tabular fixed point defined in~\Cref{eq:Qstar_tabular_optimal}. By~\Cref{lem:bellman_difference_selector}, the nonlinear Bellman difference can be written as a linear mode selected by a stochastic policy. The resulting SLS model associated with standard Q-learning is summarized next.

\begin{lemma}\label{lem:standard_q_modes}
For the standard tabular deterministic Q-learning recursion \eqref{eq:standard_tabular_Q} with $Q_k\in\R^n$, there exists a stochastic policy $\mu_k:\Sset\to\Delta_{|\Aset|}$ such that
\begin{equation}\label{eq:standard_q_error_sls}
    Q_{k+1}-Q^\star=A^Q_{\mu_k}(Q_k-Q^\star),
\end{equation}
where
\[
    A^Q_{\mu}=I-\eta D(I-\gamma P\Gamma^\mu)\in\R^{n\times n}.
\]
In particular, the corresponding deterministic-policy mode matrices are
\[
    A^Q_{\pi}=I-\eta D(I-\gamma P\Gamma^\pi),\qquad \pi\in\Theta.
\]
\end{lemma}
\begin{proof}
The proof can be found in \cite{lee2026lyapunovcertified}. We omit the details.
\end{proof}

We now formally introduce the switching family corresponding to the SLS model in \eqref{eq:standard_q_error_sls}.

\begin{definition}\label{def:AQ_switching_family}
The switching family associated with standard tabular Q-learning and deterministic policies is defined as
\[
    \mathcal A^Q:=\{A^Q_\pi:\pi\in\Theta\}\subset\R^{n\times n}.
\]
\end{definition}
For every stochastic policy $\mu:\Sset\to\Delta_{|\Aset|}$, the corresponding stochastic mode satisfies $A^Q_\mu\in\co(\mathcal A^Q)$. Thus the standard Q-learning error recursion in \eqref{eq:standard_q_error_sls} switches over $\co(\mathcal A^Q)$.
The SLS model in \eqref{eq:standard_q_error_sls} and the switching family in~\Cref{def:AQ_switching_family} give a JSR-based convergence certificate for standard deterministic Q-learning.
\begin{lemma}\label{lem:standard_q_deterministic_convergence}
Let $Q_k\in\R^n$ be generated by the standard deterministic Q-learning recursion \eqref{eq:standard_tabular_Q}, and suppose that
\[
    \rho(\mathcal A^Q)<1,
\]
where $\mathcal A^Q$ is defined in~\Cref{def:AQ_switching_family}. Then $Q_k\to Q^\star$. More precisely, for every $\varepsilon>0$ such that
\[
    \beta_\varepsilon:=\rho(\mathcal A^Q)+\varepsilon<1,
\]
the JSR Lyapunov norm $p_\varepsilon$ from~\Cref{lem:common_lyapunov_construction}, applied to the family $\mathcal A^Q$, satisfies
\[
    p_\varepsilon(Q_k-Q^\star)\le \beta_\varepsilon^k p_\varepsilon(Q_0-Q^\star),\qquad k\ge0.
\]
If $C_\varepsilon$ is the corresponding norm-equivalence constant, then
\[
    \norm{Q_k-Q^\star}_2\le \sqrt{C_\varepsilon}\,\beta_\varepsilon^k\norm{Q_0-Q^\star}_2,
    \qquad k\ge0.
\]
\end{lemma}
\begin{proof}
The result follows by applying the JSR convergence implication in~\Cref{lem:standard_jsr_convergence} to the standard Q-learning switching family; see \cite{lee2026lyapunovcertified} for the same direct-switching argument. We omit the details.
\end{proof}

\section{Deterministic dueling Q-learning}\label{sec:tabular_preconditioner}
The dueling architecture was introduced by Wang et al. \cite{wang2016dueling} for deep Q-networks (DQNs) \cite{mnih2015humanlevel}, and a tabular analysis of related dueling Q-learning methods was recently given by Daley et al. \cite{daley2025analysis}. The algorithm studied here is therefore not proposed as a new algorithmic template; rather, the goal is to give a systematic SLS analysis of its centered tabular form.

\subsection{Common/differential decomposition}\label{sec:common_differential_decomposition}
We first collect the common/differential decomposition needed for the analysis of tabular dueling Q-learning. For a fixed state $s$, write the action block as
\[
    Q_k(s,\cdot)=\begin{bmatrix}Q_k(s,a_1)\\ \vdots\\ Q_k(s,a_{|\Aset|})\end{bmatrix}\in\R^{|\Aset|}.
\]
Define the state-wise action-mean lift operator $\Pi$ by
\begin{equation}\label{eq:mean_projection}
    (\Pi Q)(s,a)=\frac{1}{|\Aset|}\sum_{b\in\Aset}Q(s,b),
    \qquad s\in\Sset,\ a\in\Aset.
\end{equation}
Equivalently, $\Pi Q$ is the action-repeated lift of the state mean $\bar V_Q(s):=|\Aset|^{-1}\sum_{b\in\Aset}Q(s,b)$. The next lemma establishes that $\Pi$ is an orthogonal projection in the stated coordinates.

\begin{lemma}\label{lem:mean_projection_matrix_form}
The state-wise action-mean operator $\Pi\in\R^{n\times n}$ can be written as the following matrix form:
\begin{equation}\label{eq:mean_projection_matrix_form}
    \Pi=\left(\frac{1}{|\Aset|}\one_{|\Aset|}\one_{|\Aset|}^\top\right)\otimes I_{|\Sset|}.
\end{equation}
Moreover, $\Pi$ is an orthogonal projection: $\Pi^2=\Pi$ and $\Pi^\top=\Pi$.
\end{lemma}
\begin{proof}
Write $Q$ in action blocks as $Q=\col\bigl(Q(\cdot,a_1),\ldots,Q(\cdot,a_{|\Aset|})\bigr)$ with $Q(\cdot,a)\in\R^{|\Sset|}$.
Let
\[
    M_{|\Aset|}:=\frac{1}{|\Aset|}\one_{|\Aset|}\one_{|\Aset|}^\top.
\]
The $a$th action block of $(M_{|\Aset|}\otimes I_{|\Sset|})Q$ is
\[
    \frac{1}{|\Aset|}\sum_{b\in\Aset}Q(\cdot,b).
\]
Therefore the coordinate indexed by $(s,a)$ is
\[
    \bigl((M_{|\Aset|}\otimes I_{|\Sset|})Q\bigr)(s,a)=\frac{1}{|\Aset|}\sum_{b\in\Aset}Q(s,b),
\]
which is exactly the definition of $\Pi Q$ in \eqref{eq:mean_projection}. This proves the matrix representation \eqref{eq:mean_projection_matrix_form}.

Next, we have
\[
    M_{|\Aset|}^2=\frac{1}{|\Aset|^2}\one\one^\top\one\one^\top
    =\frac{1}{|\Aset|}\one\one^\top=M_{|\Aset|},
    \qquad M_{|\Aset|}^\top=M_{|\Aset|}.
\]
Using $(B\otimes C)(D\otimes E)=BD\otimes CE$ and $(B\otimes C)^\top=B^\top\otimes C^\top$, we get
\[
    \Pi^2=(M_{|\Aset|}\otimes I_{|\Sset|})(M_{|\Aset|}\otimes I_{|\Sset|})=M_{|\Aset|}^2\otimes I_{|\Sset|}^2=M_{|\Aset|}\otimes I_{|\Sset|}=\Pi,
\]
\[
    \Pi^\top=(M_{|\Aset|}\otimes I_{|\Sset|})^\top=M_{|\Aset|}^\top\otimes I_{|\Sset|}^\top=M_{|\Aset|}\otimes I_{|\Sset|}=\Pi.
\]
Thus $\Pi$ is symmetric and idempotent, hence an orthogonal projection.
\end{proof}

Since $\Pi$ is an orthogonal projection, the orthogonal projection onto the orthogonal complement of the state-wise action-common subspace is
\[
    \Pi_\perp:=I-\Pi.
\]
The two associated subspaces are defined as
\[
    \Vsub:=\range(\Pi)=\{Q\in\R^{|\Sset||\Aset|}:Q(s,a)=Q(s,b)\text{ for every }s\text{ and all }a,b\},
\]
\[
    \Asub:=\range(\Pi_\perp)=\left\{Q\in\R^{|\Sset||\Aset|}:\sum_{a\in\Aset}Q(s,a)=0\text{ for every }s\right\},
\]
where $\Vsub$ will be called the V-space and $\Asub$ will be called the A-space.
Therefore, $\Pi Q$ is the state-wise action-common component, and $\Pi_\perp Q$ is the action-differential component. The corresponding product space is
\[
\Vsub\times\Asub=\range(\Pi)\times\range(\Pi_\perp)=\{(V,A)\in\R^n\times\R^n:V=\Pi V,\ A=\Pi_\perp A\},
\]
which will be called AV-space. This form makes the coupling between the two components explicit and allows the same convergence analysis to be stated on $\Vsub\times\Asub$. Consequently, every $Q$-vector has the unique decomposition
\begin{equation}\label{eq:tab_VA_statewise_decomposition}
    Q=\Pi Q+\Pi_\perp Q.
\end{equation}
Using this decomposition, each Q-function iterate $Q_k$ can be written as
\begin{equation}\label{eq:tab_a_centered}
    Q_k=V_k+A_k,
    \qquad V_k=\Pi Q_k,
    \qquad A_k=\Pi_\perp Q_k.
\end{equation}
Equivalently, $V_k$ is constant across actions within each state block and
\begin{equation}\label{eq:CQ}
    \sum_{a\in\Aset}A_k(s,a)=0,
    \qquad s\in\Sset.
\end{equation}
Similarly, the optimal action-value vector can be obtained as
\[
    Q^\star=V^\star+A^\star,
    \qquad V^\star:=\Pi Q^\star,
    \qquad A^\star:=\Pi_\perp Q^\star.
\]
In these coordinates, standard deterministic Q-learning with step-size $\eta$ applies the same gain to both subspaces. Indeed, using~\eqref{eq:tab_VA_statewise_decomposition}, the standard Q-learning recursion
\[
    Q_{k+1}=Q_k+\eta D\bigl(F(Q_k)-Q_k\bigr)
\]
can be written as
\[
    Q_{k+1}=Q_k+C^Q D\bigl(F(Q_k)-Q_k\bigr),
\]
where $C^Q$ is the preconditioning matrix defined as
\begin{equation}\label{eq:centered_va_product_space}
    C^Q=\eta I=\eta\Pi+\eta\Pi_\perp.
\end{equation}
Therefore, the standard update does not distinguish the state-wise action-common component from the action-differential component.

\subsection{Deterministic dueling Q-learning in centered components}\label{sec:va_space_mean_updates}
Before studying sample-based stochastic dueling Q-learning, we first analyze its deterministic counterpart in the centered components from \eqref{eq:tab_a_centered}. This deterministic setting focuses on the mean Bellman-residual dynamics and makes the coupling between the state-wise action-common component in V-space and the action-differential component in A-space explicit.

For the dueling Q-learning method, define $Q_k=V_k+A_k$ and use the Bellman residual
\[
    F(Q_k)-Q_k=R+\gamma P V_{Q_k}-Q_k.
\]
In this paper, we consider the following deterministic tabular dueling Q-learning update
\begin{align}
    V_{k+1}&=V_k+|\Aset|\alpha\Pi D\bigl(F(Q_k)-Q_k\bigr),\label{eq:tabular_value_update}\\
    A_{k+1}&=A_k+\beta\Pi_\perp D\bigl(F(Q_k)-Q_k\bigr).\label{eq:tabular_advantage_update}
\end{align}
The update preserves the centeredness condition in~\eqref{eq:CQ}. The factor $|\Aset|$ in \eqref{eq:tabular_value_update} appears because $\Pi$ averages over the actions in each state block; hence $|\Aset|\Pi$ converts the action average into the state-wise sum used by the value component. We can easily show that if $A_k\in\range(\Pi_\perp)$, then $A_{k+1}\in\range(\Pi_\perp)$. The vector-form statement of~\eqref{eq:tabular_value_update}--\eqref{eq:tabular_advantage_update} is summarized in~\Cref{alg:centered_va_vector}. For reference, a coordinate-by-coordinate version of~\Cref{alg:centered_va_vector} is also provided in~\Cref{alg:centered_va_elementwise} in Appendix.
\begin{algorithm}[t]
\caption{Deterministic dueling Q-learning: vector form}\label{alg:centered_va_vector}
\begin{algorithmic}[1]
\STATE Initialize $Q_0\in\R^n$, $V_0=\Pi Q_0$, and $A_0=\Pi_\perp Q_0$.
\FOR{$k=0,1,\ldots$}
\STATE Form $Q_k\leftarrow V_k+A_k$.
\STATE $V_{k+1}\leftarrow V_k+|\Aset|\alpha\Pi D\bigl(F(Q_k)-Q_k\bigr)$.
\STATE $A_{k+1}\leftarrow A_k+\beta\Pi_\perp D\bigl(F(Q_k)-Q_k\bigr)$.
\STATE $Q_{k+1}\leftarrow V_{k+1}+A_{k+1}$.
\ENDFOR
\end{algorithmic}
\end{algorithm}
To analyze~\Cref{alg:centered_va_vector}, we use the SLS framework introduced in~\Cref{sec:switching_systems}. The next lemma gives the SLS of dueling Q-learning obtained after subtracting the Bellman fixed point and introduces the associated switching family.

\begin{lemma}\label{lem:va_block_exact_sls}
For a stochastic policy $\mu:\Sset\to\Delta_{|\Aset|}$, define the block operator $B_\mu^{VA}:\Vsub\times\Asub\to\Vsub\times\Asub$, represented by a matrix in $\R^{2n\times 2n}$, by
\begin{equation}\label{eq:VA_block_coupled_sls}
    B_\mu^{VA}:=\begin{bmatrix}
    \Pi-|\Aset|\alpha\Pi D(I-\gamma P\Gamma^\mu)\Pi
    & -|\Aset|\alpha\Pi D(I-\gamma P\Gamma^\mu)\Pi_\perp\\
    -\beta\Pi_\perp D(I-\gamma P\Gamma^\mu)\Pi
    & \Pi_\perp-\beta\Pi_\perp D(I-\gamma P\Gamma^\mu)\Pi_\perp
    \end{bmatrix}.
\end{equation}
For a deterministic policy $\pi\in\Theta$, write $B_\pi^{VA}$ for this operator with $\mu=\pi$, and define the deterministic-policy block switching family for dueling Q-learning by
\begin{equation}\label{eq:VA_block_mode}
    \mathcal B^{VA}(\alpha,\beta):=\{B_\pi^{VA}:\pi\in\Theta\}\subset\R^{2n\times 2n}.
\end{equation}
Let $Q_k=V_k+A_k$ be generated by \eqref{eq:tabular_value_update}--\eqref{eq:tabular_advantage_update}, and let $V^\star:=\Pi Q^\star$ and $A^\star:=\Pi_\perp Q^\star$. Then, for each $k\ge0$, there exists a stochastic policy $\mu_k:\Sset\to\Delta_{|\Aset|}$ such that
\begin{equation}\label{eq:BVA_switching_family}
    \begin{bmatrix}V_{k+1}-V^\star\\ A_{k+1}-A^\star\end{bmatrix}
    =B_{\mu_k}^{VA}\begin{bmatrix}V_k-V^\star\\ A_k-A^\star\end{bmatrix}.
\end{equation}
\end{lemma}
\begin{proof}
Since $Q^\star=F(Q^\star)$, subtracting the fixed-point decomposition $Q^\star=V^\star+A^\star$ from \eqref{eq:tabular_value_update}--\eqref{eq:tabular_advantage_update} gives
\[
    V_{k+1}-V^\star=(V_k-V^\star)+|\Aset|\alpha\Pi D\bigl(F(Q_k)-F(Q^\star)-(Q_k-Q^\star)\bigr),
\]
\[
    A_{k+1}-A^\star=(A_k-A^\star)+\beta\Pi_\perp D\bigl(F(Q_k)-F(Q^\star)-(Q_k-Q^\star)\bigr).
\]
By~\Cref{lem:bellman_difference_selector}, there exists a stochastic policy $\mu_k:\Sset\to\Delta_{|\Aset|}$ such that
\[
    F(Q_k)-F(Q^\star)=\gamma P\Gamma^{\mu_k}(Q_k-Q^\star).
\]
Using
\[
    V_k-V^\star=\Pi(Q_k-Q^\star),\qquad A_k-A^\star=\Pi_\perp(Q_k-Q^\star),
\]
and $Q_k-Q^\star=(V_k-V^\star)+(A_k-A^\star)$, the two preceding equations give \eqref{eq:BVA_switching_family}.
\end{proof}
The following elementary identity explains why the action-common input column of the block system does not depend on the selected policy. Once a vector is constant across actions in each state, every policy selector returns the same state-wise vector.
\begin{lemma}\label{lem:selector_common_projection}
For any stochastic policies $\mu,\nu:\Sset\to\Delta_{|\Aset|}$,
\[
    \Gamma^\mu\Pi=\Gamma^\nu\Pi.
\]
\end{lemma}
\begin{proof}
Fix $x\in\R^n$. By the definition of the action-common projection,
\[
    (\Pi x)(s,a)=\frac{1}{|\Aset|}\sum_{b\in\Aset}x(s,b),
    \qquad s\in\Sset,
    \quad a\in\Aset.
\]
Thus the quantity $(\Pi x)(s,a)$ is independent of $a$ for each fixed state $s$. Hence, for any stochastic policy $\mu$,
\[
\begin{aligned}
    (\Gamma^\mu\Pi x)(s)
    &=\sum_{a\in\Aset}\mu(a\mid s)(\Pi x)(s,a) \\
    &=\sum_{a\in\Aset}\mu(a\mid s)\frac{1}{|\Aset|}\sum_{b\in\Aset}x(s,b) \\
    &=\left(\sum_{a\in\Aset}\mu(a\mid s)\right)\frac{1}{|\Aset|}\sum_{b\in\Aset}x(s,b) \\
    &=\frac{1}{|\Aset|}\sum_{b\in\Aset}x(s,b).
\end{aligned}
\]
The same calculation with $\nu$ gives
\[
    (\Gamma^\nu\Pi x)(s)=\frac{1}{|\Aset|}\sum_{b\in\Aset}x(s,b).
\]
Therefore $(\Gamma^\mu\Pi x)(s)=(\Gamma^\nu\Pi x)(s)$ for every $s\in\Sset$ and every $x\in\R^n$, which proves $\Gamma^\mu\Pi=\Gamma^\nu\Pi$.
\end{proof}
The following lemma isolates the policy-dependence of the block operator. It shows that the first block column
\[
    \begin{bmatrix}
    \Pi-|\Aset|\alpha\Pi D(I-\gamma P\Gamma^\mu)\Pi\\[0.2em]
    -\beta\Pi_\perp D(I-\gamma P\Gamma^\mu)\Pi
    \end{bmatrix}
\]
is fixed, while the second block column
\[
    \begin{bmatrix}
    -|\Aset|\alpha\Pi D(I-\gamma P\Gamma^\mu)\Pi_\perp\\[0.2em]
    \Pi_\perp-\beta\Pi_\perp D(I-\gamma P\Gamma^\mu)\Pi_\perp
    \end{bmatrix}
\]
can switch with the selected policy.
\begin{lemma}\label{lem:va_block_constant_switching_parts}
Write the block operator in \eqref{eq:VA_block_coupled_sls} as
\[
    B_\mu^{VA}=\begin{bmatrix}B_{\mu,11}^{VA}&B_{\mu,12}^{VA}\\ B_{\mu,21}^{VA}&B_{\mu,22}^{VA}\end{bmatrix}.
\]
For stochastic policies $\mu,\nu:\Sset\to\Delta_{|\Aset|}$, the first-column blocks in~\eqref{eq:VA_block_mode} are policy-independent:
\[
    B_{\mu,11}^{VA}=B_{\nu,11}^{VA},\qquad B_{\mu,21}^{VA}=B_{\nu,21}^{VA}
\]
for all stochastic policies $\mu,\nu:\Sset\to\Delta_{|\Aset|}$. The second-column blocks
\[
    B_{\mu,12}^{VA}=-|\Aset|\alpha\Pi D(I-\gamma P\Gamma^\mu)\Pi_\perp,
    \qquad
    B_{\mu,22}^{VA}=\Pi_\perp-\beta\Pi_\perp D(I-\gamma P\Gamma^\mu)\Pi_\perp
\]
can depend on the selected policy, and their differences are
\[
    B_{\mu,12}^{VA}-B_{\nu,12}^{VA}=|\Aset|\alpha\gamma\Pi D P(\Gamma^\mu-\Gamma^\nu)\Pi_\perp,
\]
\[
    B_{\mu,22}^{VA}-B_{\nu,22}^{VA}=\beta\gamma\Pi_\perp D P(\Gamma^\mu-\Gamma^\nu)\Pi_\perp.
\]
Thus the block SLS has a fixed action-common input column and, in general, switching action-differential input blocks.
\end{lemma}
\begin{proof}
The policy-dependent part of the first block column in \eqref{eq:VA_block_coupled_sls} is $P\Gamma^\mu\Pi$. By~\Cref{lem:selector_common_projection}, $\Gamma^\mu\Pi=\Gamma^\nu\Pi$ for all stochastic policies $\mu,\nu:\Sset\to\Delta_{|\Aset|}$. Therefore
\[
\begin{aligned}
    B_{\mu,11}^{VA}
    &=\Pi-|\Aset|\alpha\Pi D(I-\gamma P\Gamma^\mu)\Pi
      =\Pi-|\Aset|\alpha\Pi D(I-\gamma P\Gamma^\nu)\Pi
      =B_{\nu,11}^{VA},\\
    B_{\mu,21}^{VA}
    &=-\beta\Pi_\perp D(I-\gamma P\Gamma^\mu)\Pi
      =-\beta\Pi_\perp D(I-\gamma P\Gamma^\nu)\Pi
      =B_{\nu,21}^{VA}.
\end{aligned}
\]
For the second block column, direct subtraction gives
\[
\begin{aligned}
    B_{\mu,12}^{VA}-B_{\nu,12}^{VA}
    &=-|\Aset|\alpha\Pi D(I-\gamma P\Gamma^\mu)\Pi_\perp
      +|\Aset|\alpha\Pi D(I-\gamma P\Gamma^\nu)\Pi_\perp\\
    &=|\Aset|\alpha\gamma\Pi DP(\Gamma^\mu-\Gamma^\nu)\Pi_\perp,
\end{aligned}
\]
and
\[
\begin{aligned}
    B_{\mu,22}^{VA}-B_{\nu,22}^{VA}
    &=\Pi_\perp-\beta\Pi_\perp D(I-\gamma P\Gamma^\mu)\Pi_\perp
      -\Pi_\perp+\beta\Pi_\perp D(I-\gamma P\Gamma^\nu)\Pi_\perp\\
    &=\beta\gamma\Pi_\perp DP(\Gamma^\mu-\Gamma^\nu)\Pi_\perp.
\end{aligned}
\]
These are the two stated formulas.
\end{proof}
The identities in~\Cref{lem:selector_common_projection,lem:va_block_constant_switching_parts} show that the exact dueling Q-learning recursion should be read as a coupled block SLS with a fixed linear value-input column and switching advantage-driven blocks, rather than as two independent recursions in a general MDP.

The associated switching family for deterministic policies introduced in~\Cref{lem:va_block_exact_sls} is enough for JSR analysis. When $\rho(\mathcal B^{VA}(\alpha,\beta))<1$, the Lyapunov argument in~\Cref{lem:standard_jsr_convergence} gives the following convergence statement on $\Vsub\times\Asub$.
\begin{lemma}\label{lem:va_block_sls_convergence}
Let $Q^\star,V^\star,A^\star\in\R^n$ satisfy $Q^\star=V^\star+A^\star$, $V^\star:=\Pi Q^\star$, and $A^\star:=\Pi_\perp Q^\star$, and let
\[
    \begin{bmatrix}V_k-V^\star\\ A_k-A^\star\end{bmatrix}\in\Vsub\times\Asub
\]
follow the block SLS \eqref{eq:BVA_switching_family}. Suppose that
\[
    \rho(\mathcal B^{VA}(\alpha,\beta))<1,
\]
where $\mathcal B^{VA}(\alpha,\beta)$ is introduced in~\Cref{lem:va_block_exact_sls}. Then the block vector converges to zero, and hence $Q_k-Q^\star\to0$. More precisely, for every $\varepsilon>0$ such that
\[
    \beta_\varepsilon:=\rho(\mathcal B^{VA}(\alpha,\beta))+\varepsilon<1,
\]
the JSR Lyapunov norm $p_\varepsilon$ from~\Cref{lem:common_lyapunov_construction}, applied to the family $\mathcal B^{VA}(\alpha,\beta)$ on $\Vsub\times\Asub$, satisfies
\begin{equation}\label{eq:va_block_sls_pepsilon_decay}
    p_\varepsilon\left(\begin{bmatrix}V_k-V^\star\\ A_k-A^\star\end{bmatrix}\right)
    \le \beta_\varepsilon^k
    p_\varepsilon\left(\begin{bmatrix}V_0-V^\star\\ A_0-A^\star\end{bmatrix}\right),\qquad k\ge0.
\end{equation}
If $C_\varepsilon$ is the corresponding norm-equivalence constant, then
\begin{equation}\label{eq:va_block_sls_euclidean_decay}
    \left\|\begin{bmatrix}V_k-V^\star\\ A_k-A^\star\end{bmatrix}\right\|_2
    \le \sqrt{C_\varepsilon}\,\beta_\varepsilon^k
    \left\|\begin{bmatrix}V_0-V^\star\\ A_0-A^\star\end{bmatrix}\right\|_2,
    \qquad k\ge0.
\end{equation}
Because $\Pi$ and $\Pi_\perp$ are orthogonal projections,
\[
    \left\|\begin{bmatrix}V_k-V^\star\\ A_k-A^\star\end{bmatrix}\right\|_2^2
    =\norm{V_k-V^\star}_2^2+\norm{A_k-A^\star}_2^2
    =\norm{Q_k-Q^\star}_2^2.
\]
\end{lemma}
\begin{proof}
Let
\[
    z_k:=\begin{bmatrix}V_k-V^\star\\ A_k-A^\star\end{bmatrix}.
\]
By~\Cref{lem:va_block_exact_sls,lem:va_block_stochastic_convex_hull}, the block recursion has the form
\[
    z_{k+1}=B_{\mu_k}^{VA}z_k,
    \qquad B_{\mu_k}^{VA}\in\co(\mathcal B^{VA}(\alpha,\beta)).
\]
The assumption $\rho(\mathcal B^{VA}(\alpha,\beta))<1$ therefore allows us to apply~\Cref{lem:standard_jsr_convergence} to the family $\mathcal B^{VA}(\alpha,\beta)$ on the product space $\Vsub\times\Asub$. Hence, for every $\varepsilon>0$ with $\beta_\varepsilon=\rho(\mathcal B^{VA}(\alpha,\beta))+\varepsilon<1$,
\[
    p_\varepsilon(z_k)\le \beta_\varepsilon^k p_\varepsilon(z_0),
\]
which is \eqref{eq:va_block_sls_pepsilon_decay}. The norm-equivalence bounds for the same JSR Lyapunov norm give
\[
    \norm{z_k}_2\le p_\varepsilon(z_k)\le \beta_\varepsilon^k p_\varepsilon(z_0)
    \le \sqrt{C_\varepsilon}\,\beta_\varepsilon^k\norm{z_0}_2,
\]
which is \eqref{eq:va_block_sls_euclidean_decay}. Since $\beta_\varepsilon<1$, $z_k\to0$.
Finally, because $V_k-V^\star=\Pi(Q_k-Q^\star)$ and $A_k-A^\star=\Pi_\perp(Q_k-Q^\star)$, and because $\Pi$ and $\Pi_\perp$ are orthogonal projections onto orthogonal subspaces,
\[
    \norm{z_k}_2^2=\norm{\Pi(Q_k-Q^\star)}_2^2+\norm{\Pi_\perp(Q_k-Q^\star)}_2^2
    =\norm{Q_k-Q^\star}_2^2.
\]
Thus $z_k\to0$ implies $Q_k-Q^\star\to0$.
\end{proof}

\subsection{Deterministic dueling Q-learning in Q-space}\label{sec:qspace_mean_updates}
The same deterministic mean recursion can be run directly in Q-space. Adding the $V$- and $A$-increments in \eqref{eq:tabular_value_update}--\eqref{eq:tabular_advantage_update} gives
\begin{equation}\label{eq:tabular_VA_Qspace}
    Q_{k+1}=Q_k+C^{VA}D\bigl(F(Q_k)-Q_k\bigr),
\end{equation}
where
\begin{equation}\label{eq:CVA}
    C^{VA}=|\Aset|\alpha\Pi+\beta\Pi_\perp.
\end{equation}
This is the induced Q-only recursion, and no separate $V$ or $A$ iterate is maintained. The same update is summarized in~\Cref{alg:centered_va_q_only_vector}.
\begin{algorithm}[t]
\caption{Dueling Q-learning: induced Q-only vector form}\label{alg:centered_va_q_only_vector}
\begin{algorithmic}[1]
\STATE Initialize $Q_0\in\R^n$.
\FOR{$k=0,1,\ldots$}
\STATE $Q_{k+1}\leftarrow Q_k+(|\Aset|\alpha\Pi+\beta\Pi_\perp)D\bigl(F(Q_k)-Q_k\bigr)$.
\ENDFOR
\end{algorithmic}
\end{algorithm}
The elementwise form of the induced Q-only recursion is given in~\Cref{alg:centered_va_q_only_elementwise} in Appendix. Note that if the centered components are needed after the Q-update, they are recovered by the decomposition in~\eqref{eq:tab_a_centered}. When the gains satisfy $|\Aset|\alpha=\beta=\eta$, the preconditioner becomes $C^{VA}=\eta(\Pi+\Pi_\perp)=\eta I=C^Q$. In this case, the Q-only recursion in~\Cref{alg:centered_va_q_only_vector} reduces exactly to the standard deterministic Q-learning in~\eqref{eq:standard_tabular_Q}.

It is immediate that~\Cref{alg:centered_va_vector} induces the Q-only update in~\Cref{alg:centered_va_q_only_vector} after the $V$- and $A$-increments are added. Although this implication is clear in the forward direction, the reverse identification requires a separate check. The next lemma checks that the two descriptions generate the same iterates.
\begin{lemma}\label{lem:centered_va_q_only_equivalence}
Suppose $Q_0\in\R^n$, and~\Cref{alg:centered_va_vector} is initialized with $V_0=\Pi Q_0\in\R^n$ and $A_0=\Pi_\perp Q_0\in\R^n$, and~\Cref{alg:centered_va_q_only_vector} is initialized with the same $Q_0$. Then the two algorithms generate the same $Q_k$-sequence. Conversely, if $(Q_k)$ is generated by~\Cref{alg:centered_va_q_only_vector} and we define
\[
    V_k:=\Pi Q_k,
    \qquad
    A_k:=\Pi_\perp Q_k,
\]
then $(V_k,A_k,Q_k)$ satisfies~\Cref{alg:centered_va_vector}. Hence~\Cref{alg:centered_va_vector,alg:centered_va_q_only_vector} are equivalent under the centered identification
\[
    Q\longleftrightarrow(\Pi Q,\Pi_\perp Q).
\]
\end{lemma}
\begin{proof}
Let $\Delta_k=F(Q_k)-Q_k$. If~\Cref{alg:centered_va_vector} is run, then
\[
    Q_{k+1}=V_{k+1}+A_{k+1}=V_k+A_k+|\Aset|\alpha\Pi D\Delta_k+\beta\Pi_\perp D\Delta_k
    =Q_k+(|\Aset|\alpha\Pi+\beta\Pi_\perp)D\Delta_k,
\]
which is exactly~\Cref{alg:centered_va_q_only_vector}.

Conversely, suppose~\Cref{alg:centered_va_q_only_vector} is run and set $V_k=\Pi Q_k$ and $A_k=\Pi_\perp Q_k$. Applying $\Pi$ to the~\Cref{alg:centered_va_q_only_vector} update gives
\[
    V_{k+1}=\Pi Q_{k+1}=\Pi Q_k+\Pi(|\Aset|\alpha\Pi+\beta\Pi_\perp)D\Delta_k=V_k+|\Aset|\alpha\Pi D\Delta_k,
\]
using $\Pi^2=\Pi$ and $\Pi\Pi_\perp=0$. Similarly, applying $\Pi_\perp$ gives
\[
    A_{k+1}=\Pi_\perp Q_{k+1}=\Pi_\perp Q_k+\Pi_\perp(|\Aset|\alpha\Pi+\beta\Pi_\perp)D\Delta_k=A_k+\beta\Pi_\perp D\Delta_k,
\]
using $\Pi_\perp\Pi=0$ and $\Pi_\perp^2=\Pi_\perp$. Finally, $Q_k=V_k+A_k$ follows from $I=\Pi+\Pi_\perp$. Hence the reconstructed variables satisfy~\Cref{alg:centered_va_vector}.
\end{proof}
The modes of the induced Q-space recursion associated with deterministic policies and the corresponding Q-space SLS are summarized next.
\begin{lemma}\label{lem:va_modes}
For a stochastic policy $\mu:\Sset\to\Delta_{|\Aset|}$, define
\begin{equation}\label{eq:tabular_AVA_stochastic}
    A_\mu^{VA}:=I-C^{VA}D(I-\gamma P\Gamma^\mu)=I+C^{VA}D(\gamma P\Gamma^\mu-I)\in\R^{n\times n}.
\end{equation}
For a deterministic policy $\pi\in\Theta$, set
\begin{equation}\label{eq:tabular_AVA}
    A_\pi^{VA}:=I-(|\Aset|\alpha\Pi+\beta\Pi_\perp)D(I-\gamma P\Gamma^\pi).
\end{equation}
The deterministic-policy switching family associated with the centered dueling Q-learning Q-space recursion with gains $(\alpha,\beta)$ is
\[
\mathcal A^{VA}(\alpha,\beta):=\{A_\pi^{VA}:\pi\in\Theta\}\subset\R^{n\times n}.
\]
For the tabular dueling Q-learning recursion \eqref{eq:tabular_VA_Qspace}, there exists a stochastic policy $\mu_k:\Sset\to\Delta_{|\Aset|}$ such that
\begin{equation}\label{eq:tabular_VA_exact_sls_Qspace}
    Q_{k+1}-Q^\star=A_{\mu_k}^{VA}(Q_k-Q^\star),
\end{equation}
where $A_{\mu_k}^{VA}$ is defined in \eqref{eq:tabular_AVA_stochastic}. Moreover, $A_{\mu_k}^{VA}\in\co(\mathcal A^{VA}(\alpha,\beta))$.
\end{lemma}
\begin{proof}
Since $Q^\star=F(Q^\star)$, subtracting \eqref{eq:Qstar_tabular_optimal} from \eqref{eq:tabular_VA_Qspace} gives
\[
    Q_{k+1}-Q^\star=(Q_k-Q^\star)+C^{VA}D\bigl(F(Q_k)-Q_k-F(Q^\star)+Q^\star\bigr)
\]
\[
    =(Q_k-Q^\star)+C^{VA}D\bigl(F(Q_k)-F(Q^\star)-(Q_k-Q^\star)\bigr).
\]
By~\Cref{lem:bellman_difference_selector}, there exists a stochastic policy $\mu_k:\Sset\to\Delta_{|\Aset|}$ such that
\[
    F(Q_k)-F(Q^\star)=\gamma P\Gamma^{\mu_k}(Q_k-Q^\star).
\]
Substitution gives
\[
    Q_{k+1}-Q^\star=(Q_k-Q^\star)+C^{VA}D\bigl(\gamma P\Gamma^{\mu_k}(Q_k-Q^\star)-(Q_k-Q^\star)\bigr)
\]
\[
    =\bigl(I-C^{VA}D(I-\gamma P\Gamma^{\mu_k})\bigr)(Q_k-Q^\star).
\]
This proves \eqref{eq:tabular_VA_exact_sls_Qspace}. The convex-hull inclusion follows directly from~\Cref{lem:policy_selector_convex_hull}, because the map
\[
    M\longmapsto I-C^{VA}D(I-\gamma M)
\]
is affine in $M=P\Gamma^\mu$. Hence $A_{\mu_k}^{VA}\in\co(\mathcal A^{VA}(\alpha,\beta))$.
\end{proof}

The next lemma connects the Q-space modes with the block modes introduced in~\Cref{lem:va_block_exact_sls}.

\begin{lemma}\label{lem:va_block_similarity}
Let $T:\R^n\to\Vsub\times\Asub$ and $T^{-1}:\Vsub\times\Asub\to\R^n$ be defined by
\[
    Tx:=\begin{bmatrix}\Pi x\\ \Pi_\perp x\end{bmatrix},
    \qquad
    T^{-1}\begin{bmatrix}u\\v\end{bmatrix}:=u+v,
    \qquad (u,v)\in\Vsub\times\Asub.
\]
Then $T$ is a linear isomorphism from $\R^n$ onto $\Vsub\times\Asub$. If $\widehat B_\pi^{VA}$ denotes the restriction of $B_\pi^{VA}$ to $\Vsub\times\Asub$, then, for every $\pi\in\Theta$,
\[
    \widehat B_\pi^{VA}=\begin{bmatrix}
    \Pi A_\pi^{VA}\Pi&\Pi A_\pi^{VA}\Pi_\perp\\
    \Pi_\perp A_\pi^{VA}\Pi&\Pi_\perp A_\pi^{VA}\Pi_\perp
    \end{bmatrix}\bigg|_{\Vsub\times\Asub}=T A_\pi^{VA}T^{-1}.
\]
Consequently,
\[
    \rho(\mathcal B^{VA}(\alpha,\beta))=\rho(\mathcal A^{VA}(\alpha,\beta)),
\]
where the JSR of $\mathcal B^{VA}(\alpha,\beta)$ may be computed either as the ambient $2n$-dimensional matrix-family JSR or as the induced JSR on $\Vsub\times\Asub$.
\end{lemma}
\begin{proof}
We first verify the block formula. Substituting \eqref{eq:tabular_AVA} into the four projected blocks gives
\[
\begin{aligned}
    \Pi A_\pi^{VA}\Pi
    &=\Pi\{I-(|\Aset|\alpha\Pi+\beta\Pi_\perp)D(I-\gamma P\Gamma^\pi)\}\Pi\\
    &=\Pi^2-\bigl(|\Aset|\alpha\Pi^2+\beta\Pi\Pi_\perp\bigr)D(I-\gamma P\Gamma^\pi)\Pi\\
    &=\Pi-|\Aset|\alpha\Pi D(I-\gamma P\Gamma^\pi)\Pi,
\end{aligned}
\]
\[
\begin{aligned}
    \Pi A_\pi^{VA}\Pi_\perp
    &=\Pi\{I-(|\Aset|\alpha\Pi+\beta\Pi_\perp)D(I-\gamma P\Gamma^\pi)\}\Pi_\perp\\
    &=\Pi\Pi_\perp-\bigl(|\Aset|\alpha\Pi^2+\beta\Pi\Pi_\perp\bigr)D(I-\gamma P\Gamma^\pi)\Pi_\perp\\
    &=-|\Aset|\alpha\Pi D(I-\gamma P\Gamma^\pi)\Pi_\perp,
\end{aligned}
\]
\[
\begin{aligned}
    \Pi_\perp A_\pi^{VA}\Pi
    &=\Pi_\perp\{I-(|\Aset|\alpha\Pi+\beta\Pi_\perp)D(I-\gamma P\Gamma^\pi)\}\Pi\\
    &=\Pi_\perp\Pi-\bigl(|\Aset|\alpha\Pi_\perp\Pi+\beta\Pi_\perp^2\bigr)D(I-\gamma P\Gamma^\pi)\Pi\\
    &=-\beta\Pi_\perp D(I-\gamma P\Gamma^\pi)\Pi,
\end{aligned}
\]
and
\[
\begin{aligned}
    \Pi_\perp A_\pi^{VA}\Pi_\perp
    &=\Pi_\perp\{I-(|\Aset|\alpha\Pi+\beta\Pi_\perp)D(I-\gamma P\Gamma^\pi)\}\Pi_\perp\\
    &=\Pi_\perp^2-\bigl(|\Aset|\alpha\Pi_\perp\Pi+\beta\Pi_\perp^2\bigr)D(I-\gamma P\Gamma^\pi)\Pi_\perp\\
    &=\Pi_\perp-\beta\Pi_\perp D(I-\gamma P\Gamma^\pi)\Pi_\perp.
\end{aligned}
\]
These four identities are exactly the four blocks in \eqref{eq:VA_block_coupled_sls} with $\mu=\pi$.

The map $T$ is bijective from $\R^n$ onto $\Vsub\times\Asub$, with inverse $T^{-1}(u,v)=u+v$. Indeed, if $u\in\Vsub$ and $v\in\Asub$, then
\[
    TT^{-1}\begin{bmatrix}u\\v\end{bmatrix}
    =\begin{bmatrix}\Pi(u+v)\\ \Pi_\perp(u+v)\end{bmatrix}
    =\begin{bmatrix}u\\v\end{bmatrix},
\]
and $T^{-1}Tx=(\Pi+\Pi_\perp)x=x$ for every $x\in\R^n$.
For $x\in\R^n$, using $x=(\Pi+\Pi_\perp)x$ and the block expression just proved,
\[
\begin{aligned}
    B_\pi^{VA}Tx
    &=\begin{bmatrix}
        \Pi A_\pi^{VA}\Pi&\Pi A_\pi^{VA}\Pi_\perp\\
        \Pi_\perp A_\pi^{VA}\Pi&\Pi_\perp A_\pi^{VA}\Pi_\perp
      \end{bmatrix}
      \begin{bmatrix}\Pi x\\ \Pi_\perp x\end{bmatrix}\\
    &=\begin{bmatrix}
        \Pi A_\pi^{VA}(\Pi+\Pi_\perp)x\\
        \Pi_\perp A_\pi^{VA}(\Pi+\Pi_\perp)x
      \end{bmatrix}
      =\begin{bmatrix}\Pi A_\pi^{VA}x\\ \Pi_\perp A_\pi^{VA}x\end{bmatrix}
      =TA_\pi^{VA}x.
\end{aligned}
\]
Therefore $\widehat B_\pi^{VA}T=TA_\pi^{VA}$ and, as operators between the isomorphic spaces $\R^n$ and $\Vsub\times\Asub$, $\widehat B_\pi^{VA}=T A_\pi^{VA}T^{-1}$. For every product,
\[
    \widehat B_{\pi_k}^{VA}\cdots \widehat B_{\pi_1}^{VA}
    =T A_{\pi_k}^{VA}\cdots A_{\pi_1}^{VA}T^{-1}.
\]
Taking norms, using finite-dimensional norm equivalence for the fixed isomorphism $T$, and then taking $k$th roots and the limit in the JSR definition gives
\[
    \rho(\{\widehat B_\pi^{VA}:\pi\in\Theta\})=\rho(\mathcal A^{VA}(\alpha,\beta)).
\]
Finally, each ambient matrix $B_\pi^{VA}$ maps $\R^n\times\R^n$ into $\Vsub\times\Asub$ and leaves $\Vsub\times\Asub$ invariant. Hence the ambient JSR of $\mathcal B^{VA}(\alpha,\beta)$ equals the restricted JSR on $\Vsub\times\Asub$: one inequality is obtained by restricting the ambient norm to the invariant subspace, and the reverse follows because after the first factor every product evolves inside $\Vsub\times\Asub$ and the first-factor norm is uniformly bounded over the finite family. This proves the stated equality.
\end{proof}

This similarity result is useful because it confirms that the block AV-space analysis and the Q-space analysis certify the same stability property. Thus no conservatism is introduced merely by choosing one coordinate description over the other; the choice only changes the representation of the modes.

Next, the Q-space switching family introduced in~\Cref{lem:va_modes} gives a direct convergence criterion for the Q-only recursion in~\Cref{alg:centered_va_q_only_vector}.
\begin{proposition}\label{prop:centered_va_q_only_vector_convergence}
Let $Q_k\in\R^n$ be generated by~\Cref{alg:centered_va_q_only_vector}, and suppose that for given step-sizes $\alpha$ and $\beta$, the following inequality holds:
\[
    \rho(\mathcal A^{VA}(\alpha,\beta))<1,
\]
where $\mathcal A^{VA}(\alpha,\beta)$ is introduced in~\Cref{lem:va_modes}. Then $Q_k\to Q^\star$. More precisely, for every $\varepsilon>0$ such that
\[
    \beta_\varepsilon:=\rho(\mathcal A^{VA}(\alpha,\beta))+\varepsilon<1,
\]
the JSR Lyapunov norm $p_\varepsilon$ from~\Cref{lem:common_lyapunov_construction}, applied to the family $\mathcal A^{VA}(\alpha,\beta)$, satisfies
\begin{equation}\label{eq:centered_va_q_only_pepsilon_decay}
    p_\varepsilon(Q_k-Q^\star)\le \beta_\varepsilon^k p_\varepsilon(Q_0-Q^\star),\qquad k\ge0.
\end{equation}
If $C_\varepsilon$ is the corresponding norm-equivalence constant, then
\begin{equation}\label{eq:centered_va_q_only_euclidean_decay}
    \norm{Q_k-Q^\star}_2\le \sqrt{C_\varepsilon}\,\beta_\varepsilon^k\norm{Q_0-Q^\star}_2,
    \qquad k\ge0.
\end{equation}
\end{proposition}
\begin{proof}
By~\Cref{lem:va_modes}, the error recursion generated by~\Cref{alg:centered_va_q_only_vector} has the exact switched form
\[
    Q_{k+1}-Q^\star=A_{\mu_k}^{VA}(Q_k-Q^\star),
    \qquad A_{\mu_k}^{VA}\in\co(\mathcal A^{VA}(\alpha,\beta)).
\]
Since $\rho(\mathcal A^{VA}(\alpha,\beta))<1$, the implication in~\Cref{lem:standard_jsr_convergence} applies with $\Hfam=\mathcal A^{VA}(\alpha,\beta)$. Therefore, for every $\varepsilon>0$ satisfying $\beta_\varepsilon=\rho(\mathcal A^{VA}(\alpha,\beta))+\varepsilon<1$,
\[
    p_\varepsilon(Q_k-Q^\star)\le \beta_\varepsilon^k p_\varepsilon(Q_0-Q^\star).
\]
This is exactly \eqref{eq:centered_va_q_only_pepsilon_decay}. The norm comparison
\[
    \norm{Q_k-Q^\star}_2\le p_\varepsilon(Q_k-Q^\star),
    \qquad
    p_\varepsilon(Q_0-Q^\star)\le \sqrt{C_\varepsilon}\norm{Q_0-Q^\star}_2
\]
gives \eqref{eq:centered_va_q_only_euclidean_decay}. Since $\beta_\varepsilon<1$, $(Q_k-Q^\star)\to0$, and hence $Q_k\to Q^\star$.
\end{proof}

The equality of JSRs allows a stability certificate to be proved either in Q-space or in VA-space from~\eqref{eq:tab_a_centered}. The following conservative condition gives an explicit gain interval under which all modes in $\mathcal A^{VA}(\alpha,\beta)$ are contractions in a common norm.
\begin{lemma}\label{lem:va_block_conservative_jsr}
Suppose that, with $d_{\min}$ and $d_{\max}$ defined in \eqref{eq:sampling_weight_extrema},
\[
    0<\beta\le\frac{1}{d_{\max}},
\]
and
\begin{equation}\label{eq:va_conservative_alpha_beta_band}
    \frac{\beta}{|\Aset|}\left(1-\frac{d_{\min}(1-\gamma)}{d_{\max}(1+\gamma)}\right)
    <\alpha<
    \frac{\beta}{|\Aset|}\left(1+\frac{d_{\min}(1-\gamma)}{d_{\max}(1+\gamma)}\right).
\end{equation}
Then the deterministic dueling Q-learning family $\mathcal A^{VA}(\alpha,\beta)$ in \eqref{eq:tabular_AVA} is a common $\ell_\infty$-contraction. More precisely,
\begin{equation}\label{eq:va_conservative_common_contraction}
    \sup_{\pi\in\Theta}\norm{A_\pi^{VA}}_\infty
    \le 1-\beta d_{\min}(1-\gamma)+\bigl||\Aset|\alpha-\beta\bigr|d_{\max}(1+\gamma)<1.
\end{equation}
Consequently,
\begin{equation}\label{eq:va_conservative_jsr_less_than_one}
    \rho(\mathcal B^{VA}(\alpha,\beta))=\rho(\mathcal A^{VA}(\alpha,\beta))<1.
\end{equation}
\end{lemma}
\begin{proof}
First rewrite the dueling Q-learning preconditioner using $\Pi_\perp=I-\Pi$:
\[
    C^{VA}=|\Aset|\alpha\Pi+\beta(I-\Pi)=\beta I+(|\Aset|\alpha-\beta)\Pi.
\]
Thus, for every deterministic policy $\pi$,
\[
    A_\pi^{VA}=I-C^{VA}D(I-\gamma P\Gamma^\pi)
    =I-\beta D(I-\gamma P\Gamma^\pi)-(|\Aset|\alpha-\beta)\Pi D(I-\gamma P\Gamma^\pi).
\]
We bound the two terms on the right separately in the $\ell_\infty$ matrix norm.

For the first term, write
\[
    I-\beta D(I-\gamma P\Gamma^\pi)=I-\beta D+\beta\gamma D P\Gamma^\pi.
\]
Because $P\Gamma^\pi$ is row-stochastic, it is entrywise nonnegative and each row sums to one. Since $0<\beta d(s,a)\le\beta d_{\max}\le1$, the diagonal part $I-\beta D$ is entrywise nonnegative. Hence the whole matrix is entrywise nonnegative. The row indexed by $(s,a)$ has sum
\[
    1-\beta d(s,a)+\beta\gamma d(s,a)=1-\beta d(s,a)(1-\gamma)\le1-\beta d_{\min}(1-\gamma),
\]
where we use the fact that $P\Gamma^\pi$ is row-stochastic, and it is entrywise nonnegative and each row sums to one.
For an entrywise nonnegative matrix, the $\ell_\infty$ norm is the maximum row sum. Therefore
\[
    \norm{I-\beta D(I-\gamma P\Gamma^\pi)}_\infty\le1-\beta d_{\min}(1-\gamma).
\]
For the second term, we will use submultiplicativity. The projection $\Pi$ averages within each state-action block, so every absolute row sum of $\Pi$ equals one, and hence $\norm{\Pi}_\infty=1$. Also
\[
    \norm{D}_\infty=d_{\max},
    \qquad
    \norm{I-\gamma P\Gamma^\pi}_\infty\le\norm{I}_\infty+\gamma\norm{P\Gamma^\pi}_\infty=1+\gamma.
\]
Consequently,
\[
\begin{aligned}
    \norm{(|\Aset|\alpha-\beta)\Pi D(I-\gamma P\Gamma^\pi)}_\infty
    &=\bigl||\Aset|\alpha-\beta\bigr|\,\norm{\Pi D(I-\gamma P\Gamma^\pi)}_\infty\\
    &\le \bigl||\Aset|\alpha-\beta\bigr|\,\norm{\Pi}_\infty\norm{D}_\infty\norm{I-\gamma P\Gamma^\pi}_\infty\\
    &\le \bigl||\Aset|\alpha-\beta\bigr|\cdot 1\cdot d_{\max}
        \bigl(\norm{I}_\infty+\gamma\norm{P\Gamma^\pi}_\infty\bigr)\\
    &=\bigl||\Aset|\alpha-\beta\bigr|d_{\max}(1+\gamma).
\end{aligned}
\]
Combining the two estimates by the triangle inequality gives, uniformly in $\pi$,
\[
\begin{aligned}
    \norm{A_\pi^{VA}}_\infty
    &=\norm{I-\beta D(I-\gamma P\Gamma^\pi)
        -( |\Aset|\alpha-\beta)\Pi D(I-\gamma P\Gamma^\pi)}_\infty\\
    &\le \norm{I-\beta D(I-\gamma P\Gamma^\pi)}_\infty
        +\norm{(|\Aset|\alpha-\beta)\Pi D(I-\gamma P\Gamma^\pi)}_\infty\\
    &\le 1-\beta d_{\min}(1-\gamma)+\bigl||\Aset|\alpha-\beta\bigr|d_{\max}(1+\gamma).
\end{aligned}
\]
The interval in \eqref{eq:va_conservative_alpha_beta_band} is exactly
\[
\begin{aligned}
    -\frac{\beta d_{\min}(1-\gamma)}{d_{\max}(1+\gamma)}
    &< |\Aset|\alpha-\beta
    <\frac{\beta d_{\min}(1-\gamma)}{d_{\max}(1+\gamma)},\\
    \text{or equivalently}\qquad
    \bigl||\Aset|\alpha-\beta\bigr|
    &<\frac{\beta d_{\min}(1-\gamma)}{d_{\max}(1+\gamma)}.
\end{aligned}
\]
Multiplying the last strict inequality by $d_{\max}(1+\gamma)$ gives
\[
    \bigl||\Aset|\alpha-\beta\bigr|d_{\max}(1+\gamma)<\beta d_{\min}(1-\gamma).
\]
Substituting this into the preceding norm estimate yields, for every $\pi\in\Theta$,
\[
\begin{aligned}
    \norm{A_\pi^{VA}}_\infty
    &\le 1-\beta d_{\min}(1-\gamma)
        +\bigl||\Aset|\alpha-\beta\bigr|d_{\max}(1+\gamma)\\
    &<1-\beta d_{\min}(1-\gamma)+\beta d_{\min}(1-\gamma)=1.
\end{aligned}
\]
Therefore
\[
    \sup_{\pi\in\Theta}\norm{A_\pi^{VA}}_\infty<1.
\]
Let
\[
    q:=\sup_{\pi\in\Theta}\norm{A_\pi^{VA}}_\infty.
\]
Then $q<1$. For any product of length $k\ge1$ with $A_{\pi_i}^{VA}\in\mathcal A^{VA}(\alpha,\beta)$,
\[
\begin{aligned}
    \norm{A_{\pi_k}^{VA}A_{\pi_{k-1}}^{VA}\cdots A_{\pi_1}^{VA}}_\infty
    &\le \prod_{i=1}^k\norm{A_{\pi_i}^{VA}}_\infty
    \le q^k.
\end{aligned}
\]
Taking the supremum over all deterministic-policy products and then the $k$th root gives
\[
    \sup_{\pi_1,\ldots,\pi_k\in\Theta}
    \norm{A_{\pi_k}^{VA}\cdots A_{\pi_1}^{VA}}_\infty^{1/k}
    \le q.
\]
Hence, by the definition of the JSR,
\[
\begin{aligned}
    \rho(\mathcal A^{VA}(\alpha,\beta))
    &=\lim_{k\to\infty}
      \sup_{\pi_1,\ldots,\pi_k\in\Theta}
      \norm{A_{\pi_k}^{VA}\cdots A_{\pi_1}^{VA}}_\infty^{1/k}\\
    &\le q
    =\sup_{\pi\in\Theta}\norm{A_\pi^{VA}}_\infty<1.
\end{aligned}
\]
Finally, the similarity result in~\Cref{lem:va_block_similarity} gives
\[
    \rho(\mathcal B^{VA}(\alpha,\beta))=\rho(\mathcal A^{VA}(\alpha,\beta))<1.
\]
This proves \eqref{eq:va_conservative_common_contraction} and \eqref{eq:va_conservative_jsr_less_than_one}.
\end{proof}

The following simple example makes the $A$-space and $V$-space rate separation explicit. It removes policy-switching complications while preserving the different gains applied to the two subspaces.
\begin{example}\label{ex:one_state}
Let us consider an MDP with one state, $\Sset=\{1\}$, and two actions, $\Aset=\{1,2\}$, with a self-loop transition and uniform sampling $D=\frac12 I_2$. For a deterministic policy $\pi_i$ that selects action $i\in\{1,2\}$, the selector matrices are
\[
    P\Gamma^{\pi_1}=\begin{bmatrix}1&0\\1&0\end{bmatrix},
    \qquad
    P\Gamma^{\pi_2}=\begin{bmatrix}0&1\\0&1\end{bmatrix}.
\]
The common and differential projections are
\[
    \Pi=\frac12\begin{bmatrix}1&1\\1&1\end{bmatrix},
    \qquad
    \Pi_\perp=\frac12\begin{bmatrix}1&-1\\-1&1\end{bmatrix}.
\]
Thus every vector $Q=(Q_1,Q_2)^\top$ has the explicit decomposition
\[
    Q=\Pi Q+\Pi_\perp Q
    =\frac{Q_1+Q_2}{2}\begin{bmatrix}1\\1\end{bmatrix}
    +\frac{Q_1-Q_2}{2}\begin{bmatrix}1\\-1\end{bmatrix}.
\]
For standard Q-learning, the two deterministic-policy modes in \eqref{eq:standard_q_error_sls} are
\[
    A_{\pi_1}^Q=\begin{bmatrix}
    1-\dfrac{\eta}{2}(1-\gamma)&0\\[0.6em]
    \dfrac{\eta\gamma}{2}&1-\dfrac{\eta}{2}
    \end{bmatrix},
    \qquad
    A_{\pi_2}^Q=\begin{bmatrix}
    1-\dfrac{\eta}{2}&\dfrac{\eta\gamma}{2}\\[0.6em]
    0&1-\dfrac{\eta}{2}(1-\gamma)
    \end{bmatrix}.
\]
Use the common/differential coordinate transform
\[
    S:=\begin{bmatrix}1&1\\1&-1\end{bmatrix},
    \qquad
    S^{-1}=\frac12\begin{bmatrix}1&1\\1&-1\end{bmatrix}.
\]
A direct calculation gives
\[
    S^{-1}A_{\pi_1}^QS=
    \begin{bmatrix}
    1-\dfrac{\eta}{2}(1-\gamma)&\dfrac{\eta\gamma}{2}\\[0.8em]
    0&1-\dfrac{\eta}{2}
    \end{bmatrix},
    \qquad
    S^{-1}A_{\pi_2}^QS=
    \begin{bmatrix}
    1-\dfrac{\eta}{2}(1-\gamma)&-\dfrac{\eta\gamma}{2}\\[0.8em]
    0&1-\dfrac{\eta}{2}
    \end{bmatrix}.
\]
Thus the action-differential coordinate contracts by the factor $|1-\eta/2|$. The action-common coordinate contracts by the factor $|1-\eta(1-\gamma)/2|$, up to the upper-right coupling from the differential coordinate. The two transformed modes are simultaneously upper triangular. For any finite family of simultaneously block upper triangular matrices, the JSR is the maximum of the JSRs of the diagonal-block families; here the diagonal blocks are scalar families. Hence
\[
    \rho(\mathcal A^Q)=\max\left\{\left|1-\frac{\eta}{2}(1-\gamma)\right|,\ \left|1-\frac{\eta}{2}\right|\right\}.
\]

For dueling Q-learning, the preconditioner and modes are
\[
    C^{VA}=2\alpha\Pi+\beta\Pi_\perp
    =\begin{bmatrix}
    \alpha+\dfrac{\beta}{2}&\alpha-\dfrac{\beta}{2}\\[0.6em]
    \alpha-\dfrac{\beta}{2}&\alpha+\dfrac{\beta}{2}
    \end{bmatrix},
\]
\[
\begin{aligned}
    A_{\pi_1}^{VA}&=I-\frac12(2\alpha\Pi+\beta\Pi_\perp)
    \left(I-\gamma\begin{bmatrix}1&0\\1&0\end{bmatrix}\right),\\
    A_{\pi_2}^{VA}&=I-\frac12(2\alpha\Pi+\beta\Pi_\perp)
    \left(I-\gamma\begin{bmatrix}0&1\\0&1\end{bmatrix}\right).
\end{aligned}
\]
Using the same transform $S$,
\[
    S^{-1}A_{\pi_1}^{VA}S=
    \begin{bmatrix}
    1-\alpha(1-\gamma)&\alpha\gamma\\[0.8em]
    0&1-\dfrac{\beta}{2}
    \end{bmatrix},
    \qquad
    S^{-1}A_{\pi_2}^{VA}S=
    \begin{bmatrix}
    1-\alpha(1-\gamma)&-\alpha\gamma\\[0.8em]
    0&1-\dfrac{\beta}{2}
    \end{bmatrix}.
\]
By the same simultaneous block upper triangular argument, the resulting JSR is
\[
    \rho(\mathcal A^{VA}(\alpha,\beta))=
    \max\left\{\left|1-\alpha(1-\gamma)\right|,\ \left|1-\frac{\beta}{2}\right|\right\}.
\]
The exact acceleration condition in this one-state model is
\[
    \rho(\mathcal A^{VA}(\alpha,\beta))<\rho(\mathcal A^Q).
\]
Equivalently,
\[
    \frac{1-\rho(\mathcal A^Q)}{1-\gamma}<\alpha<\frac{1+\rho(\mathcal A^Q)}{1-\gamma},
    \qquad
    2\bigl(1-\rho(\mathcal A^Q)\bigr)<\beta<2\bigl(1+\rho(\mathcal A^Q)\bigr).
\]
For the common practical regime $0<\eta\le1$ and $\gamma$ close to one, standard Q-learning is typically bottlenecked by the action-common factor
\[
    1-\frac{\eta}{2}(1-\gamma).
\]
Choosing
\[
    \beta=\eta,
    \qquad
    \alpha=\eta
\]
keeps the action-differential factor at $1-\eta/2$ but changes the action-common factor to
\[
    1-\eta(1-\gamma).
\]
Thus, in this example, dueling Q-learning accelerates the slow action-common component by a factor of about $2$ in time-constant terms, while leaving the action-differential factor unchanged. To illustrate the same rate effect on iterates, we simulate the deterministic one-state mean recursions with reward vector $R=(0,1)^\top$, initialization $Q_0=(0,0)^\top$, $\gamma=0.9$, and $\eta=\alpha=\beta=0.2$. The optimal vector is $Q^\star=(9,10)^\top$. The resulting trajectory comparison is shown in~\Cref{fig:ex41_error}.
\begin{figure}[t]
\centering
\includegraphics[width=0.78\textwidth]{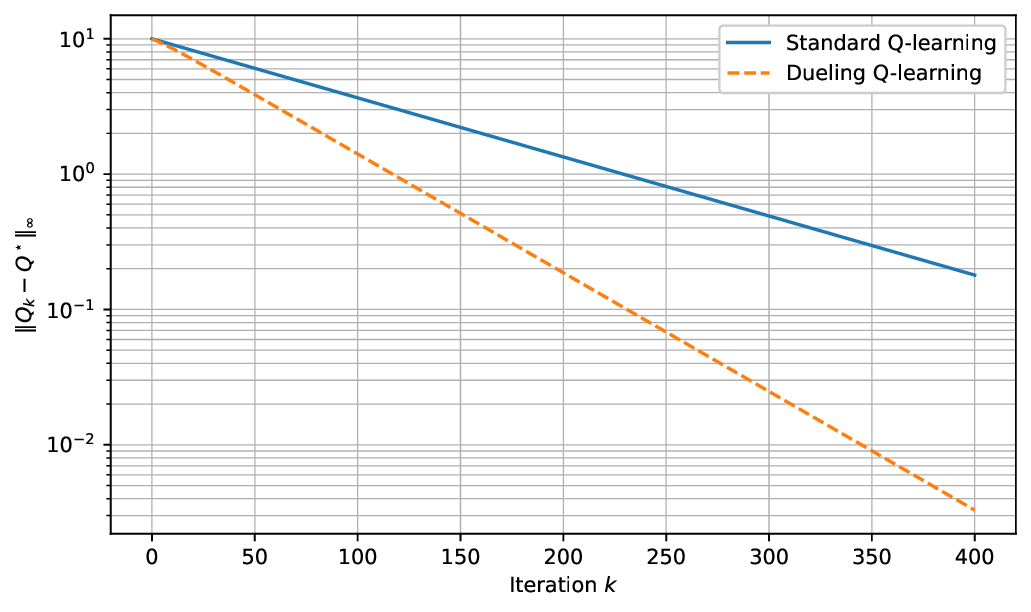}
\caption{Deterministic simulation comparing standard Q-learning and dueling Q-learning.}
\label{fig:ex41_error}
\end{figure}
\end{example}

\section{Stochastic dueling Q-learning}\label{sec:stochastic_dueling_q_learning}
This section introduces and analyzes the stochastic RL version of the deterministic dueling Q-learning recursion studied above. To keep the formulas transparent, we focus on the i.i.d. sampling model. The same conditional-mean and martingale-difference decomposition can be combined with Markovian-observation stochastic-approximation arguments, such as those in \cite{chen2024lyapunov,lee2026lyapunovcertified}, to extend the analysis beyond i.i.d. samples. At each time $k$, a state-action pair $(s_k,a_k)$ is sampled independently with
\[
    \Prob(s_k=s,a_k=a)=d(s,a):=p(s)b(a\mid s),\qquad (s,a)\in\Sset\times\Aset,
\]
where $p$ is a state-sampling distribution and $b$ is a behavior policy. Then $s'_k\sim P(\cdot\mid s_k,a_k)$ and $r_{k+1}:=r(s_k,a_k,s'_k)$ are sampled. Let $\{\mathcal F_k\}_{k\ge0}$ be the natural filtration
\[
    \mathcal F_0:=\sigma(Q_0),\qquad
    \mathcal F_k:=\sigma\bigl(Q_0,\{(s_t,a_t,s'_t,r_{t+1}):0\le t\le k-1\}\bigr),\quad k\ge1.
\]
Then $Q_k$ is $\mathcal F_k$-measurable. For the sampled transition, define
\[
    \delta_k:=r_{k+1}+\gamma\max_{c\in\Aset}Q_k(s'_k,c)-Q_k(s_k,a_k),
    \qquad
    \zeta_k:=(e_{a_k}\otimes e_{s_k})\delta_k.
\]

\subsection{AV-update}\label{sec:stochastic_va_update}

For the dueling Q-learning in AV-space, the sampled update is
\[
\begin{aligned}
    V_{k+1}(s,b)&=V_k(s,b)+\alpha\mathbf 1\{s=s_k\}\delta_k,\\
    A_{k+1}(s,b)&=A_k(s,b)+\beta\mathbf 1\{s=s_k\}
    \left(\mathbf 1\{b=a_k\}-\frac{1}{|\Aset|}\right)\delta_k,
\end{aligned}
\qquad s\in\Sset,\ b\in\Aset.
\]
The corresponding sampled implementation in AV-space is given in~\Cref{alg:centered_va_iid}.
\begin{algorithm}[t]
\caption{Dueling Q-learning in centered components}\label{alg:centered_va_iid}
\begin{algorithmic}[1]
\STATE Initialize $Q_0\in\R^n$, $V_0=\Pi Q_0$, and $A_0=\Pi_\perp Q_0$.
\FOR{$k=0,1,\ldots$}
\STATE Form $Q_k\leftarrow V_k+A_k$.
\STATE Draw $(s_k,a_k)\sim d$, draw $s'_k\sim P(\cdot\mid s_k,a_k)$, and set $r_{k+1}\leftarrow r(s_k,a_k,s'_k)$.
\STATE $\delta_k\leftarrow r_{k+1}+\gamma\max_{b\in\Aset}Q_k(s'_k,b)-Q_k(s_k,a_k)$.
\STATE Set $V_{k+1}\leftarrow V_k$ and $A_{k+1}\leftarrow A_k$.
\FOR{every $b\in\Aset$}
\STATE $V_{k+1}(s_k,b)\leftarrow V_k(s_k,b)+\alpha\delta_k$.
\ENDFOR
\FOR{every $b\in\Aset$}
\STATE $A_{k+1}(s_k,b)\leftarrow A_k(s_k,b)+\beta\bigl(\mathbf 1\{b=a_k\}-1/|\Aset|\bigr)\delta_k$.
\ENDFOR
\STATE $Q_{k+1}\leftarrow V_{k+1}+A_{k+1}$.
\ENDFOR
\end{algorithmic}
\end{algorithm}
\begin{remark}\label{rem:daley_dueling_update}
When $\beta=\alpha$, the sampled increments in~\Cref{alg:centered_va_iid} agree with the tabular dueling update derived in \cite[Section~4.1]{daley2025analysis}. In the notation of this paper, for the sampled pair $(s_k,a_k)$ and temporal-difference (TD) error $\delta_k$, that update adds
\[
    A(s_k,a)\leftarrow A(s_k,a)+\alpha\left(\mathbf 1\{a=a_k\}-\frac{1}{|\Aset|}\right)\delta_k,
    \qquad a\in\Aset,
\]
and
\[
    V(s_k)\leftarrow V(s_k)+\alpha\delta_k.
\]
The remaining difference is the stored advantage variable. Daley et al. maintain an advantage variable $\widetilde A$ that need not be centered and reconstruct
\[
    Q(s,a)=V(s)+\widetilde A(s,a)-\frac{1}{|\Aset|}\sum_{c\in\Aset}\widetilde A(s,c),
\]
so the state-wise mean of $\widetilde A$ does not affect $Q$. In contrast, the update in~\Cref{alg:centered_va_iid} maintains the centered representative $A=\Pi_\perp Q$, satisfies $\sum_{a\in\Aset}A(s,a)=0$, and reconstructs $Q=V+A$. Therefore, with centered initialization and $\beta=\alpha$, both algorithms coincide, while the internal $A$ variables follow different storage conventions.
\end{remark}

Equivalently, the update in the AV-space is
\[
\begin{bmatrix}V_{k+1}\\ A_{k+1}\end{bmatrix}
    =\begin{bmatrix}V_k\\ A_k\end{bmatrix}
    +\begin{bmatrix}|\Aset|\alpha\Pi\\ \beta\Pi_\perp\end{bmatrix}\zeta_k.
\]
Since
\[
\E[\zeta_k\mid\mathcal F_k]=D\bigl(F(Q_k)-Q_k\bigr),
\]
the martingale-difference noise is
\begin{equation}\label{eq:centered_va_iid_noise}
    w_k:=\zeta_k-D\bigl(F(Q_k)-Q_k\bigr),
\end{equation}
and satisfies
\begin{equation}\label{eq:centered_va_iid_mds}
    \E[w_k\mid\mathcal F_k]=0.
\end{equation}
Therefore, the VA-space stochastic form is
\[
\begin{bmatrix}V_{k+1}\\ A_{k+1}\end{bmatrix}
    =\begin{bmatrix}V_k\\ A_k\end{bmatrix}
    +\begin{bmatrix}|\Aset|\alpha\Pi\\ \beta\Pi_\perp\end{bmatrix}
    \{D(F(Q_k)-Q_k)+w_k\}.
\]
The conditional mean of~\Cref{alg:centered_va_iid} is the deterministic dueling Q-learning recursion \eqref{eq:tabular_value_update}--\eqref{eq:tabular_advantage_update}; the statement and proof are deferred to~\Cref{lem:centered_va_iid_conditional_mean} in Appendix.

\subsection{Q-update}\label{sec:stochastic_q_update}

Adding the two component increments in the VA-update gives the sampled Q-space update
\[
Q_{k+1}(s,b)=Q_k(s,b)+\mathbf 1\{s=s_k\}
    \left[\alpha+\beta\left(\mathbf 1\{b=a_k\}-\frac{1}{|\Aset|}\right)\right]\delta_k,
    \qquad s\in\Sset,\ b\in\Aset.
\]
Equivalently, the stochastic dueling Q-learning update can be written in Q-space as
\[
Q_{k+1}=Q_k+C^{VA}\zeta_k.
\]
Using the same martingale-difference noise $w_k$ from \eqref{eq:centered_va_iid_noise}, the Q-space stochastic form is
\begin{equation}\label{eq:centered_va_iid_qspace_noise_form}
    Q_{k+1}=Q_k+C^{VA}\{D(F(Q_k)-Q_k)+w_k\}.
\end{equation}
Subtracting $Q^\star=F(Q^\star)$ and using the measurable selection in~\Cref{lem:bellman_difference_selector}, there exists an $\mathcal F_k$-measurable stochastic policy $\mu_k:\Sset\to\Delta_{|\Aset|}$ such that
\[
    Q_{k+1}-Q^\star=A_{\mu_k}^{VA}(Q_k-Q^\star)+C^{VA}w_k,
    \qquad A_{\mu_k}^{VA}\in\co(\mathcal A^{VA}(\alpha,\beta)).
\]
Consequently,
\[
\E[Q_{k+1}-Q^\star\mid\mathcal F_k]=A_{\mu_k}^{VA}(Q_k-Q^\star).
\]
The Q-update can be applied directly to $Q_k$ without maintaining the $V$ and $A$ components as in~\Cref{alg:centered_va_q_only_iid}.
\begin{algorithm}[t]
\caption{Dueling Q-learning in Q-space}\label{alg:centered_va_q_only_iid}
\begin{algorithmic}[1]
\STATE Initialize $Q_0\in\R^n$.
\FOR{$k=0,1,\ldots$}
\STATE Draw $(s_k,a_k)\sim d$, draw $s'_k\sim P(\cdot\mid s_k,a_k)$, and set $r_{k+1}\leftarrow r(s_k,a_k,s'_k)$.
\STATE $\delta_k\leftarrow r_{k+1}+\gamma\max_{c\in\Aset}Q_k(s'_k,c)-Q_k(s_k,a_k)$.
\STATE Set $Q_{k+1}(s,a)\leftarrow Q_k(s,a)$ for every $(s,a)\in\Sset\times\Aset$.
\FOR{every $b\in\Aset$}
\STATE $Q_{k+1}(s_k,b)\leftarrow Q_k(s_k,b)+\left[\alpha+\beta\bigl(\mathbf 1\{b=a_k\}-1/|\Aset|\bigr)\right]\delta_k$.
\ENDFOR
\ENDFOR
\end{algorithmic}
\end{algorithm}

\subsubsection{Finite-time error bound analysis}
We next give a finite-time convergence estimate for the sampled recursion in~\Cref{alg:centered_va_iid}. The analysis of~\Cref{alg:centered_va_q_only_iid} is analogous and is therefore omitted. Since~\Cref{alg:centered_va_iid} uses constant gains, the conclusion is convergence to a first-moment neighborhood of $Q^\star$; the size of this neighborhood goes to zero when the common scalar gain goes to zero. Fix reference gains $\bar\alpha>0$ and $\bar\beta>0$. For $0<\alpha\le\bar\alpha$, we set
\begin{equation}\label{eq:centered_va_iid_fixed_ratio}
    \beta=\frac{\bar\beta}{\bar\alpha}\alpha,
    \qquad
    C^{VA}=\alpha\left(|\Aset|\Pi+\frac{\bar\beta}{\bar\alpha}\Pi_\perp\right).
\end{equation}
For the finite-time estimates below, assume that the reference gains satisfy
\begin{equation}\label{eq:centered_va_iid_jsr_condition}
    \rho(\mathcal A^{VA}(\bar\alpha,\bar\beta))<1,
\end{equation}
where \eqref{eq:tabular_AVA} is formed with $\alpha=\bar\alpha$ and $\beta=\bar\beta$. This keeps the two gains in a fixed ratio and lets the scalar $\alpha$ play the role of the small gain in the stochastic approximation.

\begin{lemma}\label{lem:centered_va_iid_error_system}
Let $Q_k\in\R^n$ be generated by~\Cref{alg:centered_va_iid} with the gains in \eqref{eq:centered_va_iid_fixed_ratio}. Define
\[
    \xi_k:=\alpha^{-1}C^{VA}w_k,
\]
where $w_k$ is defined in \eqref{eq:centered_va_iid_noise}. Then, for each $k\ge0$, there exists an $\mathcal F_k$-measurable stochastic policy $\mu_k:\Sset\to\Delta_{|\Aset|}$ such that
\begin{equation}\label{eq:centered_va_iid_error_system}
    Q_{k+1}-Q^\star=A_{\mu_k}^{VA}(Q_k-Q^\star)+\alpha\xi_k,
    \qquad \E[\xi_k\mid\mathcal F_k]=0.
\end{equation}
\end{lemma}
\begin{proof}
By \eqref{eq:centered_va_iid_qspace_noise_form} and \eqref{eq:centered_va_iid_fixed_ratio},
\[
    Q_{k+1}=Q_k+C^{VA}D(F(Q_k)-Q_k)+\alpha\xi_k.
\]
Subtract $Q^\star=F(Q^\star)$ and apply the measurable selection in~\Cref{lem:bellman_difference_selector}. This gives an $\mathcal F_k$-measurable stochastic policy $\mu_k:\Sset\to\Delta_{|\Aset|}$ such that
\[
    F(Q_k)-F(Q^\star)=\gamma P\Gamma^{\mu_k}(Q_k-Q^\star).
\]
Therefore
\[
    Q_{k+1}-Q^\star=(Q_k-Q^\star)+C^{VA}D\{F(Q_k)-F(Q^\star)-(Q_k-Q^\star)\}+\alpha\xi_k
\]
\[
    =\bigl(I-C^{VA}D(I-\gamma P\Gamma^{\mu_k})\bigr)(Q_k-Q^\star)+\alpha\xi_k.
\]
The matrix in the last equation is exactly $A_{\mu_k}^{VA}$ from \eqref{eq:tabular_AVA_stochastic}. The martingale identity follows from \eqref{eq:centered_va_iid_mds}.
\end{proof}

We now derive the finite-time error bound. The estimate separates the deterministic contraction from the martingale term and makes explicit the $O(\sqrt{\alpha})$ first-moment neighborhood induced by constant gains.

\begin{theorem}\label{thm:centered_va_iid_small_alpha_error}
Suppose that the JSR condition in \eqref{eq:centered_va_iid_jsr_condition} holds, and let $c$ and $K$ be the constants from~\Cref{lem:centered_va_iid_product_bound} in Appendix. Let $Q^\star\in\R^n$ be the optimal fixed point in \eqref{eq:Qstar_tabular_optimal}. Let $L$, $\sigma_0$, $\sigma_1$, $C_0$, $C_1$, and $\alpha_0$ be defined as in~\Cref{lem:centered_va_iid_mode_difference,lem:centered_va_iid_noise_growth,lem:centered_va_iid_ms_bound} in Appendix. Then, for every initial vector $Q_0\in\R^n$, every $0<\alpha\le\alpha_0$ with $\beta=(\bar\beta/\bar\alpha)\alpha$, and every $k\ge0$, the stochastic dueling Q-learning recursion in~\Cref{alg:centered_va_iid} satisfies
\begin{equation}\label{eq:centered_va_iid_small_alpha_error}
\begin{aligned}
    \E[\norm{Q_k-Q^\star}_2]
    &\le K\left(1-\frac{c\alpha}{2}\right)^k\norm{Q_0-Q^\star}_2\\
    &\quad +\alpha\gamma L K^2 k\left(1-\frac{c\alpha}{2}\right)^{k-1}\norm{Q_0-Q^\star}_2\\
    &\quad +\left(1+\frac{2\gamma K L}{c}\right)K\sqrt{\frac{2\alpha}{c}}
    \left(\sigma_0^2+\sigma_1^2\bigl(2C_0\norm{Q_0-Q^\star}_2^2+2C_1\alpha\sigma_0^2\bigr)\right)^{1/2},
\end{aligned}
\end{equation}
with the convention that $k(1-c\alpha/2)^{k-1}=0$ when $k=0$. In particular, for fixed $Q_0$, the last term in \eqref{eq:centered_va_iid_small_alpha_error} is $O(\sqrt{\alpha})$ as $\alpha\downarrow0$.
\end{theorem}
\begin{proof}
By~\Cref{lem:centered_va_iid_error_system},
\begin{equation}\label{eq:main_iid_error_recursion}
    Q_{k+1}-Q^\star=A_{\mu_k}^{VA}(Q_k-Q^\star)+\alpha\xi_k,
    \qquad \E[\xi_k\mid\mathcal F_k]=0.
\end{equation}
Every stochastic-policy mode $A_{\mu_k}^{VA}$ belongs to $\co(\mathcal A^{VA}(\alpha,\beta))$. Set
\begin{equation}\label{eq:main_iid_beta_def}
    \beta_\alpha:=1-\frac{c}{2}\alpha.
\end{equation}
Since $1-c\alpha<\beta_\alpha<1$, the product bound in~\Cref{lem:centered_va_iid_product_bound} in Appendix implies that every product of length $\ell$ from the convexified family satisfies
\begin{equation}\label{eq:main_iid_product_beta_bound}
    \norm{M_{\ell-1}\cdots M_0}_2\le K(1-c\alpha)^\ell\le K\beta_\alpha^\ell,
    \qquad \ell\ge0.
\end{equation}
Fix a stochastic policy $\bar\mu:\Sset\to\Delta_{|\Aset|}$. Introduce the fixed-policy reference filter
\begin{equation}\label{eq:main_iid_reference_filter}
    y_{k+1}=A_{\bar\mu}^{VA}y_k+\alpha\xi_k,
    \qquad y_0=Q_0-Q^\star.
\end{equation}
Write $y_k=\bar y_k+\tilde y_k$, where
\begin{equation}\label{eq:main_iid_reference_split}
\begin{aligned}
    \bar y_{k+1}&=A_{\bar\mu}^{VA}\bar y_k,
    &\bar y_0&=Q_0-Q^\star,\\
    \tilde y_{k+1}&=A_{\bar\mu}^{VA}\tilde y_k+\alpha\xi_k,
    &\tilde y_0&=0.
\end{aligned}
\end{equation}
The product bound in \eqref{eq:main_iid_product_beta_bound} gives
\begin{equation}\label{eq:main_iid_bary_bound}
\begin{aligned}
    \bar y_k
    &=(A_{\bar\mu}^{VA})^k(Q_0-Q^\star),\\
    \norm{\bar y_k}_2
    &\le \norm{(A_{\bar\mu}^{VA})^k}_2\norm{Q_0-Q^\star}_2
    \le K\beta_\alpha^k\norm{Q_0-Q^\star}_2.
\end{aligned}
\end{equation}
Let
\begin{equation}\label{eq:main_iid_gap_def}
    e_k:=Q_k-Q^\star-y_k.
\end{equation}
Then $e_0=0$. Subtracting \eqref{eq:main_iid_reference_filter} from \eqref{eq:main_iid_error_recursion} and using \eqref{eq:main_iid_gap_def} gives
\begin{equation}\label{eq:main_iid_gap_recursion}
\begin{aligned}
    e_{k+1}
    &=(Q_{k+1}-Q^\star)-y_{k+1}\\
    &=A_{\mu_k}^{VA}(Q_k-Q^\star)+\alpha\xi_k-\bigl(A_{\bar\mu}^{VA}y_k+\alpha\xi_k\bigr)\\
    &=A_{\mu_k}^{VA}(Q_k-Q^\star)-A_{\bar\mu}^{VA}y_k\\
    &=A_{\mu_k}^{VA}e_k+(A_{\mu_k}^{VA}-A_{\bar\mu}^{VA})y_k.
\end{aligned}
\end{equation}
Split $e_k=u_k+v_k$ according to $y_k=\bar y_k+\tilde y_k$ by setting
\begin{equation}\label{eq:main_iid_uv_recursions}
\begin{aligned}
    u_{k+1}&=A_{\mu_k}^{VA}u_k+(A_{\mu_k}^{VA}-A_{\bar\mu}^{VA})\bar y_k,
    &u_0&=0,\\
    v_{k+1}&=A_{\mu_k}^{VA}v_k+(A_{\mu_k}^{VA}-A_{\bar\mu}^{VA})\tilde y_k,
    &v_0&=0.
\end{aligned}
\end{equation}
For $k=0$, the sums below are empty. For $k\ge1$, unrolling the first recursion in \eqref{eq:main_iid_uv_recursions} gives
\begin{equation}\label{eq:main_iid_u_unroll}
    u_k=\sum_{t=0}^{k-1}A_{\mu_{k-1}}^{VA}\cdots A_{\mu_{t+1}}^{VA}(A_{\mu_t}^{VA}-A_{\bar\mu}^{VA})\bar y_t,
\end{equation}
where the product is the identity when $t=k-1$. The unrolled formula \eqref{eq:main_iid_u_unroll}, the product bound \eqref{eq:main_iid_product_beta_bound}, the mode-difference bound in~\Cref{lem:centered_va_iid_mode_difference} in Appendix, and the bound on $\bar y_t$ in \eqref{eq:main_iid_bary_bound} imply, for each summand,
\begin{equation}\label{eq:main_iid_u_summand_bound}
\begin{aligned}
&\norm{A_{\mu_{k-1}}^{VA}\cdots A_{\mu_{t+1}}^{VA}(A_{\mu_t}^{VA}-A_{\bar\mu}^{VA})\bar y_t}_2\\
&\quad\le \norm{A_{\mu_{k-1}}^{VA}\cdots A_{\mu_{t+1}}^{VA}}_2
      \norm{A_{\mu_t}^{VA}-A_{\bar\mu}^{VA}}_2\norm{\bar y_t}_2\\
&\quad\le K\beta_\alpha^{k-1-t}\,\alpha\gamma L\,\norm{\bar y_t}_2\\
&\quad\le K\beta_\alpha^{k-1-t}\,\alpha\gamma L\,K\beta_\alpha^t\norm{Q_0-Q^\star}_2.
\end{aligned}
\end{equation}
Summing \eqref{eq:main_iid_u_summand_bound} over $t=0,\ldots,k-1$ yields
\begin{equation}\label{eq:main_iid_u_bound}
\begin{aligned}
    \norm{u_k}_2
    &\le \sum_{t=0}^{k-1}K\beta_\alpha^{k-1-t}\alpha\gamma L K\beta_\alpha^t\norm{Q_0-Q^\star}_2\\
    &=\alpha\gamma LK^2\sum_{t=0}^{k-1}\beta_\alpha^{k-1}\norm{Q_0-Q^\star}_2\\
    &=\alpha\gamma LK^2 k\beta_\alpha^{k-1}\norm{Q_0-Q^\star}_2.
\end{aligned}
\end{equation}
We now bound the martingale part $\tilde y_k$. From \eqref{eq:main_iid_reference_split}, its explicit form is
\begin{equation}\label{eq:main_iid_tildey_unroll}
    \tilde y_k=\alpha\sum_{t=0}^{k-1}(A_{\bar\mu}^{VA})^{k-1-t}\xi_t.
\end{equation}
Expanding the squared norm in \eqref{eq:main_iid_tildey_unroll} gives
\begin{equation}\label{eq:main_iid_tildey_square_expand}
    \E[\norm{\tilde y_k}_2^2]
    =\alpha^2\sum_{s=0}^{k-1}\sum_{t=0}^{k-1}
    \E\!\bigl[\xi_s^\top\{(A_{\bar\mu}^{VA})^{k-1-s}\}^\top(A_{\bar\mu}^{VA})^{k-1-t}\xi_t\bigr].
\end{equation}
If $s<t$, then the factor multiplying $\xi_t$ in \eqref{eq:main_iid_tildey_square_expand} is $\mathcal F_t$-measurable, and therefore
\begin{equation}\label{eq:main_iid_tildey_cross_zero}
\begin{aligned}
&\E\!\bigl[\xi_s^\top\{(A_{\bar\mu}^{VA})^{k-1-s}\}^\top(A_{\bar\mu}^{VA})^{k-1-t}\xi_t\bigr]\\
&\quad=\E\!\left[\xi_s^\top\{(A_{\bar\mu}^{VA})^{k-1-s}\}^\top(A_{\bar\mu}^{VA})^{k-1-t}\E[\xi_t\mid\mathcal F_t]\right]=0.
\end{aligned}
\end{equation}
The case $t<s$ is identical after conditioning on $\mathcal F_s$. Thus \eqref{eq:main_iid_tildey_square_expand} reduces to its diagonal part:
\begin{equation}\label{eq:main_iid_tildey_diagonal}
\begin{aligned}
    \E[\norm{\tilde y_k}_2^2]
    &=\alpha^2\sum_{t=0}^{k-1}
      \E\!\left[\xi_t^\top\{(A_{\bar\mu}^{VA})^{k-1-t}\}^\top
      (A_{\bar\mu}^{VA})^{k-1-t}\xi_t\right]\\
    &=\alpha^2\sum_{t=0}^{k-1}
      \E\!\left[\norm{(A_{\bar\mu}^{VA})^{k-1-t}\xi_t}_2^2\right].
\end{aligned}
\end{equation}
Let
\begin{equation}\label{eq:main_iid_Malpha_def}
    M_\alpha:=2C_0\norm{Q_0-Q^\star}_2^2+2C_1\alpha\sigma_0^2.
\end{equation}
By~\Cref{lem:centered_va_iid_noise_growth,lem:centered_va_iid_ms_bound} in Appendix, for every $0\le t\le k$,
\begin{equation}\label{eq:main_iid_noise_second_moment_uniform}
    \E[\norm{\xi_t}_2^2]=\E\bigl[\E[\norm{\xi_t}_2^2\mid\mathcal F_t]\bigr]
    \le\sigma_0^2+\sigma_1^2\sup_{0\le j\le t}\E[\norm{Q_j-Q^\star}_2^2]
    \le\sigma_0^2+\sigma_1^2M_\alpha.
\end{equation}
Consequently, \eqref{eq:main_iid_tildey_diagonal}, \eqref{eq:main_iid_product_beta_bound}, and \eqref{eq:main_iid_noise_second_moment_uniform} imply
\begin{equation}\label{eq:main_iid_tildey_second_bound}
\begin{aligned}
    \E[\norm{\tilde y_k}_2^2]
    &\le \alpha^2\sum_{t=0}^{k-1}
      \norm{(A_{\bar\mu}^{VA})^{k-1-t}}_2^2\E[\norm{\xi_t}_2^2]\\
    &\le \alpha^2 K^2\sum_{t=0}^{k-1}\beta_\alpha^{2(k-1-t)}(\sigma_0^2+\sigma_1^2M_\alpha)\\
    &=\alpha^2 K^2(\sigma_0^2+\sigma_1^2M_\alpha)
      \sum_{j=0}^{k-1}\beta_\alpha^{2j}\\
    &\le \frac{\alpha^2K^2}{1-\beta_\alpha^2}(\sigma_0^2+\sigma_1^2M_\alpha).
\end{aligned}
\end{equation}
Jensen's inequality applied to \eqref{eq:main_iid_tildey_second_bound} gives
\begin{equation}\label{eq:main_iid_tildey_first_bound}
\begin{aligned}
    \E[\norm{\tilde y_k}_2]
    &\le \left(\E[\norm{\tilde y_k}_2^2]\right)^{1/2}\\
    &\le \frac{\alpha K}{\sqrt{1-\beta_\alpha^2}}(\sigma_0^2+\sigma_1^2M_\alpha)^{1/2}.
\end{aligned}
\end{equation}
The remaining switching correction is obtained by unrolling the second recursion in \eqref{eq:main_iid_uv_recursions}:
\begin{equation}\label{eq:main_iid_v_unroll}
    v_k=\sum_{t=0}^{k-1}A_{\mu_{k-1}}^{VA}\cdots A_{\mu_{t+1}}^{VA}(A_{\mu_t}^{VA}-A_{\bar\mu}^{VA})\tilde y_t.
\end{equation}
Using \eqref{eq:main_iid_v_unroll}, \eqref{eq:main_iid_product_beta_bound}, the mode-difference bound in~\Cref{lem:centered_va_iid_mode_difference} in Appendix, and \eqref{eq:main_iid_tildey_first_bound}, we obtain
\begin{equation}\label{eq:main_iid_v_first_bound}
\begin{aligned}
    \E[\norm{v_k}_2]
    &\le \sum_{t=0}^{k-1}
      \E\!\left[\norm{A_{\mu_{k-1}}^{VA}\cdots A_{\mu_{t+1}}^{VA}(A_{\mu_t}^{VA}-A_{\bar\mu}^{VA})\tilde y_t}_2\right]\\
    &\le \alpha\gamma LK\sum_{t=0}^{k-1}\beta_\alpha^{k-1-t}\E[\norm{\tilde y_t}_2]\\
    &\le \alpha\gamma LK\sum_{t=0}^{k-1}\beta_\alpha^{k-1-t}
      \frac{\alpha K}{\sqrt{1-\beta_\alpha^2}}(\sigma_0^2+\sigma_1^2M_\alpha)^{1/2}\\
    &\le \frac{\alpha\gamma LK}{1-\beta_\alpha}
      \frac{\alpha K}{\sqrt{1-\beta_\alpha^2}}(\sigma_0^2+\sigma_1^2M_\alpha)^{1/2}.
\end{aligned}
\end{equation}
From \eqref{eq:main_iid_gap_def}, \eqref{eq:main_iid_reference_split}, and $e_k=u_k+v_k$, we have
\begin{equation}\label{eq:main_iid_total_decomp}
    Q_k-Q^\star=\bar y_k+\tilde y_k+u_k+v_k.
\end{equation}
The decomposition \eqref{eq:main_iid_total_decomp} and the bounds in \eqref{eq:main_iid_bary_bound}, \eqref{eq:main_iid_u_bound}, \eqref{eq:main_iid_tildey_first_bound}, and \eqref{eq:main_iid_v_first_bound} give
\begin{equation}\label{eq:main_iid_pre_final_bound}
\begin{aligned}
    \E[\norm{Q_k-Q^\star}_2]
    &\le \norm{\bar y_k}_2+\E[\norm{\tilde y_k}_2]+\norm{u_k}_2+\E[\norm{v_k}_2]\\
    &\le K\beta_\alpha^k\norm{Q_0-Q^\star}_2+
    \alpha\gamma LK^2 k\beta_\alpha^{k-1}\norm{Q_0-Q^\star}_2\\
    &\quad +\frac{\alpha K}{\sqrt{1-\beta_\alpha^2}}(\sigma_0^2+\sigma_1^2M_\alpha)^{1/2}\\
    &\quad +\frac{\alpha\gamma LK}{1-\beta_\alpha}
      \frac{\alpha K}{\sqrt{1-\beta_\alpha^2}}(\sigma_0^2+\sigma_1^2M_\alpha)^{1/2}\\
    &= K\beta_\alpha^k\norm{Q_0-Q^\star}_2+
    \alpha\gamma LK^2 k\beta_\alpha^{k-1}\norm{Q_0-Q^\star}_2\\
    &\quad +\left(1+\frac{\alpha\gamma LK}{1-\beta_\alpha}\right)
    \frac{\alpha K}{\sqrt{1-\beta_\alpha^2}}(\sigma_0^2+\sigma_1^2M_\alpha)^{1/2}.
\end{aligned}
\end{equation}
By \eqref{eq:main_iid_beta_def},
\begin{equation}\label{eq:main_iid_beta_denominator_bounds}
    1-\beta_\alpha=\frac{c\alpha}{2},
    \qquad
    1-\beta_\alpha^2=(1-\beta_\alpha)(1+\beta_\alpha)\ge\frac{c\alpha}{2}.
\end{equation}
Therefore \eqref{eq:main_iid_beta_denominator_bounds} implies
\begin{equation}\label{eq:main_iid_prefactor_bounds}
    1+\frac{\alpha\gamma LK}{1-\beta_\alpha}=1+\frac{2\gamma KL}{c},
    \qquad
    \frac{\alpha K}{\sqrt{1-\beta_\alpha^2}}\le K\sqrt{\frac{2\alpha}{c}}.
\end{equation}
Substituting \eqref{eq:main_iid_prefactor_bounds} into \eqref{eq:main_iid_pre_final_bound} gives
\begin{equation}\label{eq:main_iid_after_substitution}
\begin{aligned}
    \E[\norm{Q_k-Q^\star}_2]
    &\le K\beta_\alpha^k\norm{Q_0-Q^\star}_2+
    \alpha\gamma LK^2 k\beta_\alpha^{k-1}\norm{Q_0-Q^\star}_2\\
    &\quad +\left(1+\frac{2\gamma KL}{c}\right)K\sqrt{\frac{2\alpha}{c}}
    (\sigma_0^2+\sigma_1^2M_\alpha)^{1/2}.
\end{aligned}
\end{equation}
Using \eqref{eq:main_iid_beta_def} and \eqref{eq:main_iid_Malpha_def} in \eqref{eq:main_iid_after_substitution} gives \eqref{eq:centered_va_iid_small_alpha_error}. For fixed $Q_0$,
\begin{equation}\label{eq:main_iid_Malpha_uniform}
    M_\alpha\le 2C_0\norm{Q_0-Q^\star}_2^2+2C_1\bar\alpha\sigma_0^2,
    \qquad 0<\alpha\le\bar\alpha.
\end{equation}
Hence \eqref{eq:main_iid_Malpha_uniform} shows that $(\sigma_0^2+\sigma_1^2M_\alpha)^{1/2}$ is bounded independently of $\alpha$, and the last term in \eqref{eq:centered_va_iid_small_alpha_error} is $O(\sqrt\alpha)$ as $\alpha\downarrow0$.
\end{proof}
By~\Cref{thm:centered_va_iid_small_alpha_error}, the sampled stochastic dueling Q-learning iterates in~\Cref{alg:centered_va_iid} track $Q^\star$ up to a first-moment neighborhood whose radius is of order $\sqrt\alpha$.

\subsubsection{Example}\label{sec:two_state_two_action_stochastic_illustration}
We next give a simple instance of dueling Q-learning in~\Cref{alg:centered_va_iid}. Let $\Sset=\{1,2\}$, $\Aset=\{1,2\}$, and $\gamma=0.9$, and use the uniform i.i.d. sampling distribution $d(s,a)=1/4$ for every $(s,a)\in\Sset\times\Aset$. The transition probabilities are $P(1\mid1,1)=0.95$, $P(2\mid1,1)=0.05$, $P(1\mid1,2)=0.05$, $P(2\mid1,2)=0.95$, $P(1\mid2,1)=0.05$, $P(2\mid2,1)=0.95$, $P(1\mid2,2)=0.95$, and $P(2\mid2,2)=0.05$.
The one-step reward is defined for each sampled next state as $r(1,1,1)=r(1,1,2)=0.0$, $r(1,2,1)=r(1,2,2)=0.2$, $r(2,1,1)=r(2,1,2)=1.0$, and $r(2,2,1)=r(2,2,2)=0.9$.
Solving the Bellman optimality equation gives $Q^\star(1,1)=7.992$, $Q^\star(1,2)=8.840$, $Q^\star(2,1)=9.640$, and $Q^\star(2,2)=8.892$, with $\pi^\star(1)=2$ and $\pi^\star(2)=1$.
Both methods are initialized at zero and run for $3000$ iterations. Standard Q-learning uses the constant step-size $\eta=0.2$. The stochastic recursion in~\Cref{alg:centered_va_iid} uses the constant step-sizes $\alpha=0.2$ and $\beta=0.2$.
Since $|\Aset|=2$, the effective common-component gain in \eqref{eq:CVA} is $|\Aset|\alpha=0.4$, while the differential gain is $\beta=0.2$. Thus this setting deliberately gives the dueling recursion a larger state-wise action-common gain than standard Q-learning, while keeping the action-differential gain equal to the Q-learning step-size. For a fair sampling comparison, the same sampled state-action and next-state sequence is used for both methods within each seed. The plotted quantity is $\norm{Q_k-Q^\star}_\infty$.
\begin{figure}[t]
\centering
\includegraphics[width=0.78\textwidth]{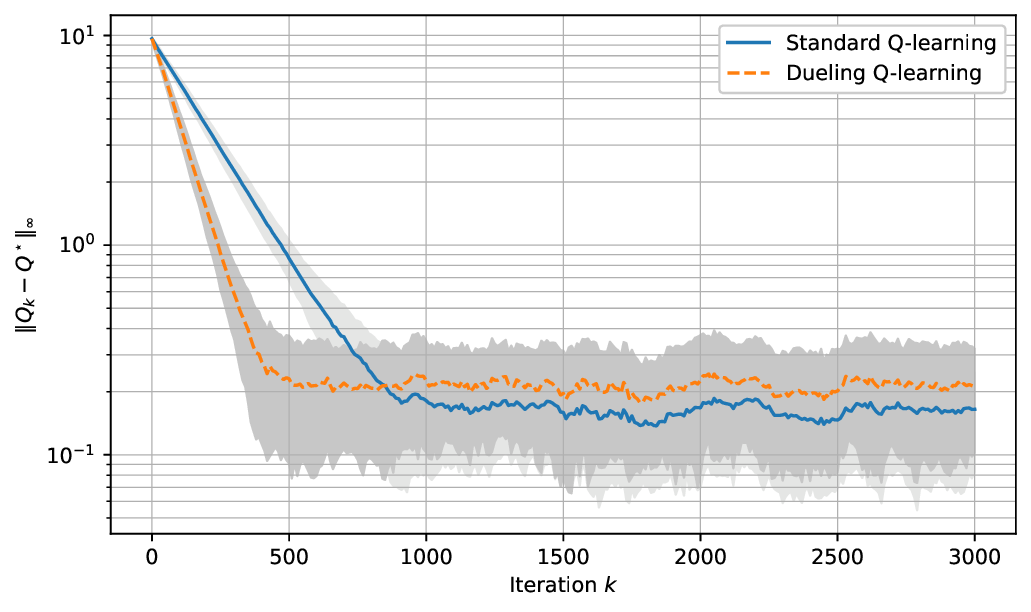}
\caption{Stochastic comparison between standard Q-learning and~\Cref{alg:centered_va_iid}. Standard Q-learning uses step-size $\eta=0.2$, and dueling Q-learning uses step-sizes $\alpha=0.2$ and $\beta=0.2$. The mean curves are computed over $100$ seeds, and the shaded bands mark one sample standard deviation from the corresponding mean.}
\label{fig:ex52_stochastic}
\end{figure}
The results in~\Cref{fig:ex52_stochastic} show that the dueling recursion reduces the error more quickly during the initial transient, while its constant-gain stochastic fluctuations are larger later in the run.

\clearpage
\appendix
\section*{Appendix}
\addcontentsline{toc}{section}{Appendix}
\section{Auxiliary switching-system and block-operator facts}\label{app:auxiliary_switching}

\begin{lemma}\label{lem:convex-hull-jsr}
For a finite matrix family $\Hfam=\{A_1,\ldots,A_M\}\subset\R^{m\times m}$, one has
\[
    \rho(\co(\Hfam))=\rho(\Hfam),
\]
where the left-hand side is computed over all convex combinations of matrices in $\Hfam$.
\end{lemma}
\begin{proof}
This is the convex-hull invariance property of the JSR; see, for example, \cite[Chapter~1]{jungers2009joint}. We omit the proof.
\end{proof}

\begin{lemma}\label{lem:similarity-jsr}
Let $S\in\R^{m\times m}$ be nonsingular. For every bounded matrix family $\Hfam\subset\R^{m\times m}$,
\[
    \rho\bigl(\{S^{-1}AS:A\in\Hfam\}\bigr)=\rho(\Hfam).
\]
\end{lemma}
\begin{proof}
This standard property follows from norm equivalence in the definition of the JSR; see, for example, \cite[Chapter~1]{jungers2009joint}. We omit the proof.
\end{proof}

\begin{lemma}\label{lem:va_block_stochastic_convex_hull}
For every stochastic policy $\mu:\Sset\to\Delta_{|\Aset|}$, the block operator introduced in~\Cref{lem:va_block_exact_sls} satisfies
\[
    B_\mu^{VA}\in\co(\mathcal B^{VA}(\alpha,\beta)).
\]
\end{lemma}
\begin{proof}
Any stochastic policy selector can be written as a convex combination of deterministic policy selectors. Thus there are weights $\lambda_\pi\ge0$, $\sum_{\pi\in\Theta}\lambda_\pi=1$, such that
\[
    P\Gamma^\mu=\sum_{\pi\in\Theta}\lambda_\pi P\Gamma^\pi.
\]
The formula \eqref{eq:VA_block_coupled_sls} defining $B_\mu^{VA}$ is affine in $P\Gamma^\mu$. Substituting the preceding convex decomposition into each block gives
\[
    B_\mu^{VA}=\sum_{\pi\in\Theta}\lambda_\pi B_\pi^{VA},
\]
which proves the claim.
\end{proof}

\section{Elementwise updates and Q-space identities}\label{app:elementwise_updates}

\begin{algorithm}[t]
\caption{Deterministic dueling Q-learning: elementwise form}\label{alg:centered_va_elementwise}
\begin{algorithmic}[1]
\STATE Initialize $Q_0\in\R^n$ and, for every $s\in\Sset$ and $a\in\Aset$, set
\[
V_0(s,a)\leftarrow \frac{1}{|\Aset|}\sum_{c\in\Aset}Q_0(s,c),
\qquad
A_0(s,a)\leftarrow Q_0(s,a)-V_0(s,a).
\]
\FOR{$k=0,1,\ldots$}
\STATE Form $Q_k\leftarrow V_k+A_k$.
\STATE For every $(s,a)\in\Sset\times\Aset$, compute the Bellman residual $F(Q_k)(s,a)-Q_k(s,a)$.
\STATE For every $s\in\Sset$ and $a\in\Aset$, set
\[
V_{k+1}(s,a)\leftarrow V_k(s,a)+\alpha\sum_{b\in\Aset}d(s,b)\bigl(F(Q_k)(s,b)-Q_k(s,b)\bigr).
\]
\STATE For every $s\in\Sset$ and $a\in\Aset$, set
\[
\begin{aligned}
A_{k+1}(s,a)\leftarrow A_k(s,a)&+\beta\Biggl[d(s,a)\bigl(F(Q_k)(s,a)-Q_k(s,a)\bigr)\\
&\qquad\qquad -\frac{1}{|\Aset|}\sum_{b\in\Aset}d(s,b)\bigl(F(Q_k)(s,b)-Q_k(s,b)\bigr)\Biggr].
\end{aligned}
\]
\STATE $Q_{k+1}\leftarrow V_{k+1}+A_{k+1}$.
\ENDFOR
\end{algorithmic}
\end{algorithm}

\begin{proposition}[Induced Q-space recursion]\label{prop:centered_va_update_identity}
Suppose that $Q_k,V_k,A_k\in\R^n$, $Q_k=V_k+A_k$, and that $V_k$ and $A_k$ follow \eqref{eq:tabular_value_update}--\eqref{eq:tabular_advantage_update}. Then the induced Q-space recursion is \eqref{eq:tabular_VA_Qspace} with the preconditioner \eqref{eq:CVA}.
\end{proposition}
\begin{proof}
Starting from the definition of the induced Q-iterate, $Q_{k+1}=V_{k+1}+A_{k+1}$, we substitute the two updates \eqref{eq:tabular_value_update}--\eqref{eq:tabular_advantage_update} to obtain
\[
    Q_{k+1}=V_k+A_k+|\Aset|\alpha\Pi D\bigl(F(Q_k)-Q_k\bigr)+\beta\Pi_\perp D\bigl(F(Q_k)-Q_k\bigr)
\]
\[
    =Q_k+(|\Aset|\alpha\Pi+\beta\Pi_\perp)D\bigl(F(Q_k)-Q_k\bigr),
\]
where the second equality uses $Q_k=V_k+A_k$. This is exactly \eqref{eq:tabular_VA_Qspace} with the preconditioner \eqref{eq:CVA}.

The same identity can be verified coordinate by coordinate. For any $z\in\R^n$, the projection definitions give
\[
    (\Pi z)(s,a)=\frac{1}{|\Aset|}\sum_{b\in\Aset}z(s,b),
    \qquad
    (\Pi_\perp z)(s,a)=z(s,a)-\frac{1}{|\Aset|}\sum_{b\in\Aset}z(s,b).
\]
Therefore
\[
\begin{aligned}
    (C^{VA}z)(s,a)
    &=|\Aset|\alpha(\Pi z)(s,a)+\beta(\Pi_\perp z)(s,a)\\
    &=\alpha\sum_{b\in\Aset}z(s,b)+\beta\left(z(s,a)-\frac{1}{|\Aset|}\sum_{b\in\Aset}z(s,b)\right).
\end{aligned}
\]
Taking $z=D(F(Q_k)-Q_k)$ gives the sum of the value increment and the centered advantage increment in \eqref{eq:tabular_value_update}--\eqref{eq:tabular_advantage_update}. Hence the coordinate and vector forms are the same recursion.
\end{proof}

\begin{algorithm}[t]
\caption{Dueling Q-learning: elementwise form of the induced Q-only recursion}\label{alg:centered_va_q_only_elementwise}
\begin{algorithmic}[1]
\STATE Initialize $Q_0\in\R^n$.
\FOR{$k=0,1,\ldots$}
\STATE For every $(s,a)\in\Sset\times\Aset$, compute the Bellman residual $F(Q_k)(s,a)-Q_k(s,a)$.
\STATE For every $s\in\Sset$ and $a\in\Aset$, set
\[
\begin{aligned}
Q_{k+1}(s,a)\leftarrow Q_k(s,a)
&+\alpha\sum_{b\in\Aset}d(s,b)\bigl(F(Q_k)(s,b)-Q_k(s,b)\bigr)\\
&+\beta\Biggl[d(s,a)\bigl(F(Q_k)(s,a)-Q_k(s,a)\bigr)\\
&\qquad\qquad -\frac{1}{|\Aset|}\sum_{b\in\Aset}d(s,b)\bigl(F(Q_k)(s,b)-Q_k(s,b)\bigr)\Biggr].
\end{aligned}
\]
\ENDFOR
\end{algorithmic}
\end{algorithm}

\section{Conditional-mean identities}\label{app:conditional_mean_identities}

\begin{lemma}\label{lem:centered_va_iid_conditional_mean}
Under the i.i.d. sampling model, the conditional mean recursion of~\Cref{alg:centered_va_iid} is exactly the deterministic dueling Q-learning recursion \eqref{eq:tabular_value_update}--\eqref{eq:tabular_advantage_update}. Equivalently, for every $s\in\Sset$ and $a\in\Aset$,
\begin{equation}\label{eq:centered_iid_mean_value}
    \E[V_{k+1}(s,a)-V_k(s,a)\mid\mathcal F_k]
    =\alpha\sum_{b\in\Aset}d(s,b)\bigl(F(Q_k)(s,b)-Q_k(s,b)\bigr)
\end{equation}
\begin{equation}\label{eq:centered_iid_mean_advantage}
\begin{aligned}
    \E[A_{k+1}(s,a)-A_k(s,a)\mid\mathcal F_k]
    &=\beta\Biggl[d(s,a)\bigl(F(Q_k)(s,a)-Q_k(s,a)\bigr)\\
    &\qquad -\frac{1}{|\Aset|}\sum_{b\in\Aset}d(s,b)\bigl(F(Q_k)(s,b)-Q_k(s,b)\bigr)\Biggr].
\end{aligned}
\end{equation}
\end{lemma}
\begin{proof}
Condition on $\mathcal F_k$ and define
\[
    \Delta_k(s,a):=F(Q_k)(s,a)-Q_k(s,a).
\]
Conditioned on $\mathcal F_k$, if the sampled pair is $(s_k,a_k)=(s,a)$, then only the next state and reward remain random. Hence
\[
\begin{aligned}
\E[\delta_k\mid\mathcal F_k,s_k=s,a_k=a]
&=\E\left[r_{k+1}+\gamma\max_{c\in\Aset}Q_k(s'_k,c)-Q_k(s,a)\mid\mathcal F_k,s_k=s,a_k=a\right]\\
&=R(s,a)+\gamma\sum_{s'\in\Sset}P(s'\mid s,a)\max_{c\in\Aset}Q_k(s',c)-Q_k(s,a)\\
&=F(Q_k)(s,a)-Q_k(s,a)=\Delta_k(s,a).
\end{aligned}
\]
Therefore
\[
    \E[\mathbf 1\{s_k=s,a_k=a\}\delta_k\mid\mathcal F_k]=d(s,a)\Delta_k(s,a).
\]
Summing this identity over the sampled action gives
\[
    \E[\mathbf 1\{s_k=s\}\delta_k\mid\mathcal F_k]=\sum_{b\in\Aset}d(s,b)\Delta_k(s,b).
\]
The value update in~\Cref{alg:centered_va_iid} adds $\alpha\delta_k$ to every action copy of the sampled state. Thus, for fixed $(s,a)$,
\[
    \E[V_{k+1}(s,a)-V_k(s,a)\mid\mathcal F_k]
    =\alpha\E[\mathbf 1\{s_k=s\}\delta_k\mid\mathcal F_k]
    =\alpha\sum_{b\in\Aset}d(s,b)\Delta_k(s,b),
\]
which is \eqref{eq:centered_iid_mean_value}. The advantage update adds $\beta(\mathbf 1\{a=a_k\}-1/|\Aset|)\delta_k$ when $s_k=s$. Hence
\[
\begin{aligned}
\E[A_{k+1}(s,a)-A_k(s,a)\mid\mathcal F_k]
&=\beta\E\left[\mathbf 1\{s_k=s\}\left(\mathbf 1\{a_k=a\}-\frac{1}{|\Aset|}\right)\delta_k\mid\mathcal F_k\right]\\
&=\beta\left[d(s,a)\Delta_k(s,a)-\frac{1}{|\Aset|}\sum_{b\in\Aset}d(s,b)\Delta_k(s,b)\right],
\end{aligned}
\]
which is \eqref{eq:centered_iid_mean_advantage}. These are exactly the coordinate forms of \eqref{eq:tabular_value_update}--\eqref{eq:tabular_advantage_update}.
\end{proof}

\section{Stochastic bounds}\label{app:stochastic_bounds}

The stochastic bounds in this appendix use the fixed-ratio gains introduced in
\eqref{eq:centered_va_iid_fixed_ratio}. Thus the relative scaling of the value
and advantage gains is held fixed while the scalar gain $\alpha$ varies:
\[
    \beta=\frac{\bar\beta}{\bar\alpha}\alpha,
    \qquad
    C^{VA}=|\Aset|\alpha\Pi+\beta\Pi_\perp
    =\alpha\left(|\Aset|\Pi+\frac{\bar\beta}{\bar\alpha}\Pi_\perp\right).
\]
Consequently,
\[
    \frac{C^{VA}}{\alpha}=|\Aset|\Pi+\frac{\bar\beta}{\bar\alpha}\Pi_\perp
\]
is independent of $\alpha$. The reference gains $(\bar\alpha,\bar\beta)$ define
the endpoint family $\mathcal A^{VA}(\bar\alpha,\bar\beta)$, and the modes for
smaller $\alpha$ are obtained by interpolating between the identity and the
corresponding reference modes.

\begin{lemma}\label{lem:centered_va_iid_fixed_ratio_interpolation}
Suppose that the JSR condition~\eqref{eq:centered_va_iid_jsr_condition} holds. Then, there exist a norm $p$ on $\R^n$ and a constant $c>0$, independent of the step-size $\alpha$, with $c\bar\alpha<1$, such that for every $0<\alpha\le\bar\alpha$,
\begin{equation}\label{eq:centered_va_iid_fixed_ratio_interpolation}
    p(Mx)\le(1-c\alpha)p(x),
    \qquad \forall x\in\R^n,
    \quad \forall M\in\co(\mathcal A^{VA}(\alpha,\beta)),
\end{equation}
whenever $\mathcal A^{VA}(\alpha,\beta)$ is formed with the fixed-ratio gains in \eqref{eq:centered_va_iid_fixed_ratio}.
\end{lemma}
\begin{proof}
Write $\bar A_\pi^{VA}$ for the mode in \eqref{eq:tabular_AVA} formed with $\alpha=\bar\alpha$ and $\beta=\bar\beta$; explicitly,
\begin{equation}\label{eq:centered_va_iid_reference_mode}
    \bar A_\pi^{VA}
    :=I-\bigl(|\Aset|\bar\alpha\Pi+\bar\beta\Pi_\perp\bigr)D(I-\gamma P\Gamma^\pi).
\end{equation}
Choose $\theta$ such that
\[
    \rho(\mathcal A^{VA}(\bar\alpha,\bar\beta))<\theta<1.
\]
By~\Cref{lem:common_lyapunov_construction}, there is a norm $p$ such that
\[
    p(\bar A_\pi^{VA}x)\le\theta p(x),\qquad x\in\R^n,
    \quad \pi\in\Theta.
\]
By convexity of the norm, the same bound holds for every $\bar M\in\co(\{\bar A_\pi^{VA}:\pi\in\Theta\})$:
\[
    p(\bar Mx)\le\theta p(x),\qquad x\in\R^n.
\]
For the gains in \eqref{eq:centered_va_iid_fixed_ratio}, the corresponding mode is
\[
\begin{aligned}
    A_\pi^{VA}
    &=I-\bigl(|\Aset|\alpha\Pi+\beta\Pi_\perp\bigr)D(I-\gamma P\Gamma^\pi)\\
    &=I-\alpha\left(|\Aset|\Pi+\frac{\bar\beta}{\bar\alpha}\Pi_\perp\right)D(I-\gamma P\Gamma^\pi)\\
    &=I-\frac{\alpha}{\bar\alpha}\bigl(|\Aset|\bar\alpha\Pi+\bar\beta\Pi_\perp\bigr)D(I-\gamma P\Gamma^\pi)\\
    &=I-\frac{\alpha}{\bar\alpha}(I-\bar A_\pi^{VA})\\
    &=\left(1-\frac{\alpha}{\bar\alpha}\right)I+\frac{\alpha}{\bar\alpha}\bar A_\pi^{VA}.
\end{aligned}
\]
Thus, if $M\in\co(\mathcal A^{VA}(\alpha,\beta))$ is formed with the same fixed-ratio gains, then there exists $\bar M\in\co(\{\bar A_\pi^{VA}:\pi\in\Theta\})$ such that
\[
    M=\left(1-\frac{\alpha}{\bar\alpha}\right)I+\frac{\alpha}{\bar\alpha}\bar M.
\]
Therefore, for every $x\in\R^n$,
\[
    p(Mx)\le\left(1-\frac{\alpha}{\bar\alpha}\right)p(x)+\frac{\alpha}{\bar\alpha}p(\bar Mx)
    \le\left(1-\frac{1-\theta}{\bar\alpha}\alpha\right)p(x).
\]
Set $c=(1-\theta)/\bar\alpha$. Since $\theta\in(0,1)$, one has $c\bar\alpha<1$, and the preceding inequality proves \eqref{eq:centered_va_iid_fixed_ratio_interpolation}.
\end{proof}

\begin{lemma}\label{lem:centered_va_iid_product_bound}
Suppose that \eqref{eq:centered_va_iid_jsr_condition} holds. Then there exist constants $c>0$ and $K<\infty$, independent of $\alpha$, with $c\bar\alpha<1$, such that for every $0<\alpha\le\bar\alpha$, every $\ell\ge0$, and every $M_0,\ldots,M_{\ell-1}\in\co(\mathcal A^{VA}(\alpha,\beta))$ corresponding to the gains in \eqref{eq:centered_va_iid_fixed_ratio},
\begin{equation}\label{eq:centered_va_iid_product_bound}
    \norm{M_{\ell-1}\cdots M_0}_2\le K(1-c\alpha)^\ell.
\end{equation}
\end{lemma}
\begin{proof}
Let $p$ and $c$ be the norm and constant from~\Cref{lem:centered_va_iid_fixed_ratio_interpolation}. Fix $0<\alpha\le\bar\alpha$, $\ell\ge0$, and matrices
\[
    M_0,\ldots,M_{\ell-1}\in\co(\mathcal A^{VA}(\alpha,\beta))
\]
corresponding to the gains in \eqref{eq:centered_va_iid_fixed_ratio}. For $\ell=0$, the product is the identity. For $\ell\ge1$, repeated application of \eqref{eq:centered_va_iid_fixed_ratio_interpolation} gives
\[
    p(M_{\ell-1}\cdots M_0x)\le(1-c\alpha)^\ell p(x),\qquad x\in\R^n.
\]
All norms on the finite-dimensional space are equivalent, so there are constants $a,b<\infty$, independent of $\alpha$, such that
\[
    \norm{x}_2\le ap(x),
    \qquad p(x)\le b\norm{x}_2,
    \qquad x\in\R^n.
\]
Consequently,
\[
    \norm{M_{\ell-1}\cdots M_0x}_2\le ap(M_{\ell-1}\cdots M_0x)\le a(1-c\alpha)^\ell p(x)
    \le ab(1-c\alpha)^\ell\norm{x}_2.
\]
Taking the supremum over $\norm{x}_2=1$ and setting $K:=ab$ proves \eqref{eq:centered_va_iid_product_bound}.
\end{proof}

\begin{lemma}\label{lem:centered_va_iid_mode_difference}
Let
\begin{equation}\label{eq:centered_va_iid_L_def}
    L:=\sup_{\mu,\nu}\left\|\frac{C^{VA}}{\alpha}DP(\Gamma^\mu-\Gamma^\nu)\right\|_2,
\end{equation}
where the supremum is over all stochastic policies $\mu,\nu:\Sset\to\Delta_{|\Aset|}$. Then $L<\infty$ and, for any stochastic policies $\mu,\nu:\Sset\to\Delta_{|\Aset|}$,
\begin{equation}\label{eq:centered_va_iid_mode_difference}
    \norm{A_\mu^{VA}-A_\nu^{VA}}_2\le\alpha\gamma L.
\end{equation}
\end{lemma}
\begin{proof}
By \eqref{eq:tabular_AVA_stochastic} and \eqref{eq:centered_va_iid_fixed_ratio}, the two modes differ only through the Bellman policy selector:
\[
    A_\mu^{VA}-A_\nu^{VA}=\alpha\gamma\frac{C^{VA}}{\alpha}DP(\Gamma^\mu-\Gamma^\nu).
\]
This gives \eqref{eq:centered_va_iid_mode_difference}. The supremum in \eqref{eq:centered_va_iid_L_def} is finite because the stochastic-policy set is a finite product of probability simplices and is compact, while the norm above is continuous in $(\mu,\nu)$.
\end{proof}

\begin{lemma}\label{lem:centered_va_iid_noise_growth}
There exist finite constants $\sigma_0,\sigma_1\ge0$, depending only on the MDP, $C^{VA}/\alpha$, and $Q^\star\in\R^n$, such that the noise term $\xi_k\in\R^n$ in~\Cref{lem:centered_va_iid_error_system} satisfies
\begin{equation}\label{eq:centered_va_iid_noise_growth}
    \E[\norm{\xi_k}_2^2\mid\mathcal F_k]
    \le \sigma_0^2+\sigma_1^2\norm{Q_k-Q^\star}_2^2,
    \qquad k\ge0.
\end{equation}
\end{lemma}
\begin{proof}
Using the reward bound in \eqref{eq:reward_bound}, for every realization of the sample at time $k$,
\[
    |\delta_k|\le R_{\max}+\gamma\max_{c\in\Aset}|Q_k(s'_k,c)|+|Q_k(s_k,a_k)|
    \le R_{\max}+(1+\gamma)\norm{Q_k}_2.
\]
Since $\zeta_k=(e_{a_k}\otimes e_{s_k})\delta_k$, the same bound holds for $\norm{\zeta_k}_2$. Also, $D(F(Q_k)-Q_k)=\E[\zeta_k\mid\mathcal F_k]$, so Jensen's inequality gives
\[
    \norm{D(F(Q_k)-Q_k)}_2\le R_{\max}+(1+\gamma)\norm{Q_k}_2.
\]
Using \eqref{eq:centered_va_iid_noise},
\[
\begin{aligned}
    \norm{w_k}_2
    &=\norm{\zeta_k-D(F(Q_k)-Q_k)}_2\\
    &\le \norm{\zeta_k}_2+\norm{D(F(Q_k)-Q_k)}_2\\
    &\le2\{R_{\max}+(1+\gamma)\norm{Q_k}_2\}.
\end{aligned}
\]
Moreover, $\norm{Q_k}_2\le\norm{Q^\star}_2+\norm{Q_k-Q^\star}_2$. Since $\xi_k=\alpha^{-1}C^{VA}w_k$,
\[
\begin{aligned}
    \norm{\xi_k}_2
    &\le \left\|\frac{C^{VA}}{\alpha}\right\|_2\norm{w_k}_2\\
    &\le 2\left\|\frac{C^{VA}}{\alpha}\right\|_2\{R_{\max}+(1+\gamma)\norm{Q_k}_2\}\\
    &\le 2\left\|\frac{C^{VA}}{\alpha}\right\|_2\{R_{\max}+(1+\gamma)\norm{Q^\star}_2\}
    +2\left\|\frac{C^{VA}}{\alpha}\right\|_2(1+\gamma)\norm{Q_k-Q^\star}_2.
\end{aligned}
\]
Squaring the preceding inequality and using $(u+v)^2\le2u^2+2v^2$ proves \eqref{eq:centered_va_iid_noise_growth}; for example, one may take
\[
    \sigma_0^2=8\left\|\frac{C^{VA}}{\alpha}\right\|_2^2\{R_{\max}+(1+\gamma)\norm{Q^\star}_2\}^2,
    \qquad
    \sigma_1^2=8\left\|\frac{C^{VA}}{\alpha}\right\|_2^2(1+\gamma)^2.
\]
Taking the conditional expectation preserves this deterministic upper bound because $Q_k-Q^\star$ is $\mathcal F_k$-measurable.
\end{proof}

\begin{lemma}\label{lem:geometric_kq_bound}
For every $q\in[0,1)$ and every integer $k\ge1$,
\begin{equation}\label{eq:geometric_kq_bound}
    kq^{k-1}\le\frac{1}{1-q}.
\end{equation}
\end{lemma}
\begin{proof}
Since $0\le q<1$, one has $q^j\ge q^{k-1}$ for $j=0,1,\ldots,k-1$. Hence
\[
    kq^{k-1}\le\sum_{j=0}^{k-1}q^j\le\sum_{j=0}^{\infty}q^j=\frac{1}{1-q}.
\]
\end{proof}

\begin{lemma}\label{lem:centered_va_iid_ms_bound}
Suppose that the fixed-ratio gains in \eqref{eq:centered_va_iid_fixed_ratio} are used and that the JSR condition \eqref{eq:centered_va_iid_jsr_condition} holds. By~\Cref{lem:centered_va_iid_product_bound}, choose constants $c>0$ and $K<\infty$, independent of $\alpha$, such that $c\bar\alpha<1$ and, for every $0<\alpha\le\bar\alpha$, every $\ell\ge0$, and every $M_0,\ldots,M_{\ell-1}\in\co(\mathcal A^{VA}(\alpha,\beta))$ corresponding to these gains,
\[
    \norm{M_{\ell-1}\cdots M_0}_2\le K(1-c\alpha)^\ell.
\]
Let $L$, $\sigma_0$, and $\sigma_1$ be the constants from~\Cref{lem:centered_va_iid_mode_difference,lem:centered_va_iid_noise_growth}. Define
\[
    C_0:=1+4K^2+4\left(\frac{\gamma L K^2}{c}\right)^2
\]
and
\[
    C_1:=4\left(\frac{K^2}{c}+\frac{\gamma^2L^2K^4}{c^3}\right).
\]
Set
\[
    \alpha_0:=\min\left\{\bar\alpha,\frac{1}{2C_1\sigma_1^2}\right\},
\]
where the second term is interpreted as $+\infty$ when $\sigma_1=0$. Then, for every initial vector $Q_0\in\R^n$, every $0<\alpha\le\alpha_0$, and every $k\ge0$, the recursion in \eqref{eq:centered_va_iid_error_system} satisfies
\begin{equation}\label{eq:centered_va_iid_ms_bound}
    \sup_{0\le t\le k}\E[\norm{Q_t-Q^\star}_2^2]
    \le 2C_0\norm{Q_0-Q^\star}_2^2+2C_1\alpha\sigma_0^2.
\end{equation}
\end{lemma}
\begin{proof}
Fix a stochastic policy $\bar\mu:\Sset\to\Delta_{|\Aset|}$. Let
\[
    y_{k+1}=A_{\bar\mu}^{VA}y_k+\alpha\xi_k,
    \qquad y_0=Q_0-Q^\star,
\]
and decompose $y_k=\bar y_k+\tilde y_k$, where
\[
    \bar y_{k+1}=A_{\bar\mu}^{VA}\bar y_k,
    \qquad \bar y_0=Q_0-Q^\star,
\]
\[
    \tilde y_{k+1}=A_{\bar\mu}^{VA}\tilde y_k+\alpha\xi_k,
    \qquad \tilde y_0=0.
\]
Let
\[
    e_k:=Q_k-Q^\star-y_k.
\]
Subtracting the fixed-policy filter from \eqref{eq:centered_va_iid_error_system} gives $e_0=0$ and
\[
\begin{aligned}
    e_{k+1}
    &=(Q_{k+1}-Q^\star)-y_{k+1}\\
    &=A_{\mu_k}^{VA}(Q_k-Q^\star)+\alpha\xi_k-\bigl(A_{\bar\mu}^{VA}y_k+\alpha\xi_k\bigr)\\
    &=A_{\mu_k}^{VA}e_k+(A_{\mu_k}^{VA}-A_{\bar\mu}^{VA})y_k\\
    &=A_{\mu_k}^{VA}e_k+(A_{\mu_k}^{VA}-A_{\bar\mu}^{VA})\bar y_k
      +(A_{\mu_k}^{VA}-A_{\bar\mu}^{VA})\tilde y_k.
\end{aligned}
\]
Accordingly, write $e_k=u_k+v_k$, where the two parts are generated by the two forcing terms:
\[
    u_{k+1}=A_{\mu_k}^{VA}u_k+(A_{\mu_k}^{VA}-A_{\bar\mu}^{VA})\bar y_k,
    \qquad u_0=0,
\]
\[
    v_{k+1}=A_{\mu_k}^{VA}v_k+(A_{\mu_k}^{VA}-A_{\bar\mu}^{VA})\tilde y_k,
    \qquad v_0=0.
\]
Since $e_k=Q_k-Q^\star-y_k$ and $y_k=\bar y_k+\tilde y_k$, adding $y_k$ to both sides yields
\[
\begin{aligned}
    Q_k-Q^\star
    &=y_k+e_k\\
    &=\bar y_k+\tilde y_k+u_k+v_k.
\end{aligned}
\]
Since $A_{\bar\mu}^{VA}\in\co(\mathcal A^{VA}(\alpha,\beta))$, the product bound in \eqref{eq:centered_va_iid_product_bound} gives
\[
\begin{aligned}
    \bar y_k
    &=(A_{\bar\mu}^{VA})^k(Q_0-Q^\star),\\
    \norm{\bar y_k}_2
    &\le K(1-c\alpha)^k\norm{Q_0-Q^\star}_2
    \le K\norm{Q_0-Q^\star}_2.
\end{aligned}
\]
The stochastic part of the fixed-policy filter has the explicit form
\[
    \tilde y_k=\alpha\sum_{t=0}^{k-1}(A_{\bar\mu}^{VA})^{k-1-t}\xi_t.
\]
Expanding the square gives
\[
    \E[\norm{\tilde y_k}_2^2]
    =\alpha^2\sum_{s=0}^{k-1}\sum_{t=0}^{k-1}
    \E\!\bigl[\xi_s^\top\{(A_{\bar\mu}^{VA})^{k-1-s}\}^\top(A_{\bar\mu}^{VA})^{k-1-t}\xi_t\bigr].
\]
If $s<t$, then the factor multiplying $\xi_t$ is $\mathcal F_t$-measurable. Therefore
\[
\begin{aligned}
&\E\!\bigl[\xi_s^\top\{(A_{\bar\mu}^{VA})^{k-1-s}\}^\top(A_{\bar\mu}^{VA})^{k-1-t}\xi_t\bigr]\\
&\quad=\E\!\left[\xi_s^\top\{(A_{\bar\mu}^{VA})^{k-1-s}\}^\top(A_{\bar\mu}^{VA})^{k-1-t}\E[\xi_t\mid\mathcal F_t]\right]=0.
\end{aligned}
\]
The case $t<s$ is identical after conditioning on $\mathcal F_s$. Thus the double sum reduces to its diagonal part:
\[
\begin{aligned}
    \E[\norm{\tilde y_k}_2^2]
    &=\alpha^2\sum_{t=0}^{k-1}
      \E\!\left[\xi_t^\top\{(A_{\bar\mu}^{VA})^{k-1-t}\}^\top
      (A_{\bar\mu}^{VA})^{k-1-t}\xi_t\right]\\
    &=\alpha^2\sum_{t=0}^{k-1}
      \E\!\left[\norm{(A_{\bar\mu}^{VA})^{k-1-t}\xi_t}_2^2\right].
\end{aligned}
\]
With
\[
    M_k:=\sup_{0\le t\le k}\E[\norm{Q_t-Q^\star}_2^2],
\]
Equations \eqref{eq:centered_va_iid_product_bound} and \eqref{eq:centered_va_iid_noise_growth} imply, for $k\ge1$,
\[
\begin{aligned}
    \E[\norm{\tilde y_k}_2^2]
    &=\alpha^2\sum_{t=0}^{k-1}
      \E\!\left[\norm{(A_{\bar\mu}^{VA})^{k-1-t}\xi_t}_2^2\right]\\
    &\le \alpha^2K^2\sum_{t=0}^{k-1}(1-c\alpha)^{2(k-1-t)}\E[\norm{\xi_t}_2^2]\\
    &\le \alpha^2K^2\sum_{t=0}^{k-1}(1-c\alpha)^{2(k-1-t)}\{\sigma_0^2+\sigma_1^2M_{k-1}\}\\
    &\le \frac{\alpha K^2}{c}\{\sigma_0^2+\sigma_1^2M_{k-1}\}.
\end{aligned}
\]
The last step follows by the change of variables $j=k-1-t$ and the geometric bound
\[
    \sum_{t=0}^{k-1}(1-c\alpha)^{2(k-1-t)}
    =\sum_{j=0}^{k-1}(1-c\alpha)^{2j}
    \le \frac{1}{1-(1-c\alpha)^2}\le\frac{1}{c\alpha},
    \qquad 0<c\alpha<1.
\]
For $k\ge1$, unrolling the recursion for $u_k$ gives
\[
    u_k=\sum_{t=0}^{k-1}A_{\mu_{k-1}}^{VA}\cdots A_{\mu_{t+1}}^{VA}(A_{\mu_t}^{VA}-A_{\bar\mu}^{VA})\bar y_t,
\]
where the product is the identity when $t=k-1$. Hence, using \eqref{eq:centered_va_iid_product_bound} and \eqref{eq:centered_va_iid_mode_difference}, each summand satisfies
\[
\begin{aligned}
&\norm{A_{\mu_{k-1}}^{VA}\cdots A_{\mu_{t+1}}^{VA}(A_{\mu_t}^{VA}-A_{\bar\mu}^{VA})\bar y_t}_2\\
&\quad\le \norm{A_{\mu_{k-1}}^{VA}\cdots A_{\mu_{t+1}}^{VA}}_2
      \norm{A_{\mu_t}^{VA}-A_{\bar\mu}^{VA}}_2\norm{\bar y_t}_2\\
&\quad\le K(1-c\alpha)^{k-1-t}\,\alpha\gamma L\,\norm{\bar y_t}_2\\
&\quad\le K(1-c\alpha)^{k-1-t}\,\alpha\gamma L\,K(1-c\alpha)^t\norm{Q_0-Q^\star}_2.
\end{aligned}
\]
Summing over $t=0,\ldots,k-1$ gives
\[
\begin{aligned}
    \norm{u_k}_2
    &\le \sum_{t=0}^{k-1}K(1-c\alpha)^{k-1-t}\alpha\gamma L K(1-c\alpha)^t\norm{Q_0-Q^\star}_2\\
    &=\alpha\gamma LK^2\sum_{t=0}^{k-1}(1-c\alpha)^{k-1}\norm{Q_0-Q^\star}_2\\
    &=\alpha\gamma LK^2 k(1-c\alpha)^{k-1}\norm{Q_0-Q^\star}_2\\
    &\le \alpha\gamma LK^2\frac{1}{c\alpha}\norm{Q_0-Q^\star}_2\\
    &=\frac{\gamma LK^2}{c}\norm{Q_0-Q^\star}_2.
\end{aligned}
\]
The last inequality follows from \Cref{lem:geometric_kq_bound} with $q=1-c\alpha$, because $0<c\alpha<1$. Similarly,
\[
    v_k=\sum_{t=0}^{k-1}A_{\mu_{k-1}}^{VA}\cdots A_{\mu_{t+1}}^{VA}(A_{\mu_t}^{VA}-A_{\bar\mu}^{VA})\tilde y_t,
\]
and each summand satisfies
\[
\begin{aligned}
&\norm{A_{\mu_{k-1}}^{VA}\cdots A_{\mu_{t+1}}^{VA}(A_{\mu_t}^{VA}-A_{\bar\mu}^{VA})\tilde y_t}_2\\
&\quad\le K(1-c\alpha)^{k-1-t}\alpha\gamma L\norm{\tilde y_t}_2.
\end{aligned}
\]
Therefore,
\[
\begin{aligned}
    \norm{v_k}_2
    &\le \sum_{t=0}^{k-1}K(1-c\alpha)^{k-1-t}\alpha\gamma L\norm{\tilde y_t}_2\\
    &=\alpha\gamma LK\sum_{t=0}^{k-1}(1-c\alpha)^{k-1-t}\norm{\tilde y_t}_2.
\end{aligned}
\]
Using the weighted Cauchy inequality,
\[
    \left(\sum_{t=0}^{k-1}q_t\norm{\tilde y_t}_2\right)^2
    \le\left(\sum_{t=0}^{k-1}q_t\right)\left(\sum_{t=0}^{k-1}q_t\norm{\tilde y_t}_2^2\right),
    \qquad q_t:=(1-c\alpha)^{k-1-t},
\]
and the preceding bound on $\E[\norm{\tilde y_t}_2^2]$, we obtain
\[
\begin{aligned}
    \E[\norm{v_k}_2^2]
    &\le \alpha^2\gamma^2L^2K^2\left(\sum_{t=0}^{k-1}q_t\right)\left(\sum_{t=0}^{k-1}q_t\E[\norm{\tilde y_t}_2^2]\right)\\
    &\le \alpha^2\gamma^2L^2K^2\left(\frac{1}{c\alpha}\right)^2\frac{\alpha K^2}{c}\{\sigma_0^2+\sigma_1^2M_{k-1}\}\\
    &=\frac{\alpha\gamma^2L^2K^4}{c^3}\{\sigma_0^2+\sigma_1^2M_{k-1}\}.
\end{aligned}
\]
Using $\norm{z_1+z_2+z_3+z_4}_2^2\le4\sum_{j=1}^4\norm{z_j}_2^2$ with $z_1=\bar y_k$, $z_2=\tilde y_k$, $z_3=u_k$, and $z_4=v_k$, we have, for every $1\le r\le k$,
\[
\begin{aligned}
    \E[\norm{Q_r-Q^\star}_2^2]
    &\le4\norm{\bar y_r}_2^2+4\E[\norm{\tilde y_r}_2^2]
      +4\norm{u_r}_2^2+4\E[\norm{v_r}_2^2]\\
    &\le4K^2\norm{Q_0-Q^\star}_2^2
      +4\left(\frac{\gamma LK^2}{c}\right)^2\norm{Q_0-Q^\star}_2^2\\
    &\quad+4\alpha\left(\frac{K^2}{c}+\frac{\gamma^2L^2K^4}{c^3}\right)
      \{\sigma_0^2+\sigma_1^2M_{k-1}\}.
\end{aligned}
\]
The case $r=0$ contributes $\norm{Q_0-Q^\star}_2^2$. Taking the supremum over $0\le r\le k$ and using the definitions of $C_0$ and $C_1$ gives
\[
\begin{aligned}
    M_k
    &\le \left\{1+4K^2+4\left(\frac{\gamma LK^2}{c}\right)^2\right\}\norm{Q_0-Q^\star}_2^2\\
    &\quad+4\alpha\left(\frac{K^2}{c}+\frac{\gamma^2L^2K^4}{c^3}\right)
      \{\sigma_0^2+\sigma_1^2M_{k-1}\}\\
    &= C_0\norm{Q_0-Q^\star}_2^2+C_1\alpha\{\sigma_0^2+\sigma_1^2M_{k-1}\}.
\end{aligned}
\]
For $0<\alpha\le\alpha_0$, one has $C_1\alpha\sigma_1^2\le1/2$ and $M_{k-1}\le M_k$. Hence
\[
    M_k\le C_0\norm{Q_0-Q^\star}_2^2+C_1\alpha\sigma_0^2+\frac{1}{2}M_k,
\]
which implies \eqref{eq:centered_va_iid_ms_bound}. For $k=0$, the same bound holds because $M_0=\norm{Q_0-Q^\star}_2^2$ and $C_0\ge1$.
\end{proof}


\begin{thebibliography}{19}

\bibitem{rota1960note}
Gian-Carlo Rota and W. Gilbert Strang.
\newblock A note on the joint spectral radius.
\newblock \emph{Indagationes Mathematicae}, 22:379--381, 1960; also in \emph{Proceedings of the Koninklijke Nederlandse Akademie van Wetenschappen, Series A}, 63:379--381, 1960.
\newblock \href{https://doi.org/10.1016/S1385-7258(60)50046-1}{doi:10.1016/S1385-7258(60)50046-1}.

\bibitem{jungers2009joint}
Rapha\"el M. Jungers.
\newblock \emph{The Joint Spectral Radius: Theory and Applications}.
\newblock Lecture Notes in Control and Information Sciences, volume 385. Springer, Berlin and Heidelberg, 2009.
\newblock \href{https://doi.org/10.1007/978-3-540-95980-9}{doi:10.1007/978-3-540-95980-9}.

\bibitem{liberzon2003switching}
Daniel Liberzon.
\newblock \emph{Switching in Systems and Control}.
\newblock Systems \& Control: Foundations \& Applications. Birkh\"auser Boston, Boston, MA, 2003.
\newblock \href{https://doi.org/10.1007/978-1-4612-0017-8}{doi:10.1007/978-1-4612-0017-8}.

\bibitem{lin2009stability}
Hai Lin and Panos J. Antsaklis.
\newblock Stability and stabilizability of switched linear systems: A survey of recent results.
\newblock \emph{IEEE Transactions on Automatic Control}, 54(2):308--322, 2009.
\newblock \href{https://doi.org/10.1109/TAC.2008.2012009}{doi:10.1109/TAC.2008.2012009}.

\bibitem{shorten2007stability}
Robert Shorten, Fabian Wirth, Oliver Mason, Kai Wulff, and Christopher King.
\newblock Stability criteria for switched and hybrid systems.
\newblock \emph{SIAM Review}, 49(4):545--592, 2007.
\newblock \href{https://doi.org/10.1137/05063516X}{doi:10.1137/05063516X}.

\bibitem{hushenzhang2010generating}
Jianghai Hu, Jinglai Shen, and Wei Zhang.
\newblock Generating functions of switched linear systems: Analysis, computation, and stability applications.
\newblock \emph{IEEE Transactions on Automatic Control}, 56(5):1059--1074, 2011.
\newblock \href{https://doi.org/10.1109/TAC.2010.2067590}{doi:10.1109/TAC.2010.2067590}.

\bibitem{watkins1992q}
Christopher J. C. H. Watkins and Peter Dayan.
\newblock Q-learning.
\newblock \emph{Machine Learning}, 8:279--292, 1992.
\newblock \href{https://doi.org/10.1007/BF00992698}{doi:10.1007/BF00992698}.

\bibitem{puterman1994markov}
Martin L. Puterman.
\newblock \emph{Markov Decision Processes: Discrete Stochastic Dynamic Programming}.
\newblock Wiley Series in Probability and Statistics. Wiley, New York, 1994.
\newblock \href{https://doi.org/10.1002/9780470316887}{doi:10.1002/9780470316887}.

\bibitem{bertsekas1996neuro}
Dimitri P. Bertsekas and John N. Tsitsiklis.
\newblock \emph{Neuro-Dynamic Programming}.
\newblock Athena Scientific, Belmont, MA, 1996.

\bibitem{sutton1998reinforcement}
Richard S. Sutton and Andrew G. Barto.
\newblock \emph{Reinforcement Learning: An Introduction}.
\newblock Second edition. MIT Press, Cambridge, MA, 2018.

\bibitem{chen2024lyapunov}
Zaiwei Chen, Siva T. Maguluri, Sanjay Shakkottai, and Karthikeyan Shanmugam.
\newblock A Lyapunov theory for finite-sample guarantees of Markovian stochastic approximation.
\newblock \emph{Operations Research}, 72(4):1352--1367, 2024. Published online October 6, 2023.
\newblock \href{https://doi.org/10.1287/opre.2022.0249}{doi:10.1287/opre.2022.0249}.

\bibitem{lee2026lyapunovcertified}
Donghwan Lee.
\newblock Lyapunov-Certified Direct Switching Theory for Q-Learning.
\newblock arXiv preprint arXiv:2604.19569v4, 2026.
\newblock \href{https://doi.org/10.48550/arXiv.2604.19569}{doi:10.48550/arXiv.2604.19569}.

\bibitem{daley2025analysis}
Brett Daley, Prabhat Nagarajan, Martha White, and Marlos C. Machado.
\newblock An analysis of action-value temporal-difference methods that learn state values.
\newblock \emph{Reinforcement Learning Journal}, 6:2619--2636, 2025. Presented at the Reinforcement Learning Conference, Edmonton, Alberta, Canada, August 5--9, 2025.
\newblock \url{https://rlj.cs.umass.edu/2025/papers/Paper309.html}.

\bibitem{mnih2015humanlevel}
Volodymyr Mnih, Koray Kavukcuoglu, David Silver, Andrei A. Rusu, Joel Veness, Marc G. Bellemare, Alex Graves, Martin Riedmiller, Andreas K. Fidjeland, Georg Ostrovski, Stig Petersen, Charles Beattie, Amir Sadik, Ioannis Antonoglou, Helen King, Dharshan Kumaran, Daan Wierstra, Shane Legg, and Demis Hassabis.
\newblock Human-level control through deep reinforcement learning.
\newblock \emph{Nature}, 518(7540):529--533, 2015.
\newblock \href{https://doi.org/10.1038/nature14236}{doi:10.1038/nature14236}.

\bibitem{wang2016dueling}
Ziyu Wang, Tom Schaul, Matteo Hessel, Hado van Hasselt, Marc Lanctot, and Nando de Freitas.
\newblock Dueling network architectures for deep reinforcement learning.
\newblock In \emph{Proceedings of the 33rd International Conference on Machine Learning}, Proceedings of Machine Learning Research, volume 48, pages 1995--2003. PMLR, 2016.
\newblock \url{https://proceedings.mlr.press/v48/wangf16.html}.


\bibitem{baird1994advantage}
Leemon C. Baird III.
\newblock Reinforcement learning in continuous time: Advantage updating.
\newblock In \emph{Proceedings of 1994 IEEE International Conference on Neural Networks (ICNN'94)}, volume 4, pages 2448--2453. IEEE, 1994.
\newblock \href{https://doi.org/10.1109/ICNN.1994.374604}{doi:10.1109/ICNN.1994.374604}.

\bibitem{wiering2005qvlambda}
Marco A. Wiering and Hado van Hasselt.
\newblock Two novel on-policy reinforcement learning algorithms based on TD($\lambda$)-methods.
\newblock In \emph{Proceedings of the IEEE International Symposium on Approximate Dynamic Programming and Reinforcement Learning (ADPRL)}, pages 280--287, 2007.
\newblock \href{https://doi.org/10.1109/ADPRL.2007.368200}{doi:10.1109/ADPRL.2007.368200}.

\bibitem{wiering2009qvfamily}
Marco A. Wiering and Hado van Hasselt.
\newblock The QV family compared to other reinforcement learning algorithms.
\newblock In \emph{2009 IEEE Symposium on Adaptive Dynamic Programming and Reinforcement Learning}, pages 101--108. IEEE, 2009.
\newblock \href{https://doi.org/10.1109/ADPRL.2009.4927532}{doi:10.1109/ADPRL.2009.4927532}.

\bibitem{tang2023valearning}
Yunhao Tang, R\'emi Munos, Mark Rowland, and Michal Valko.
\newblock VA-learning as a more efficient alternative to Q-learning.
\newblock In \emph{Proceedings of the 40th International Conference on Machine Learning}, Proceedings of Machine Learning Research, volume 202, pages 33739--33757. PMLR, 2023.
\newblock \url{https://proceedings.mlr.press/v202/tang23h.html}.

\end{thebibliography}
\end{document}